\documentclass{article}

\usepackage{microtype}
\usepackage{graphicx}
\usepackage{subfigure}
\usepackage{booktabs} 
\usepackage[export]{adjustbox}
\usepackage{multirow}
\usepackage{array}
\usepackage[table]{xcolor}
\usepackage{xcolor}
\usepackage[export]{adjustbox}
\usepackage{enumitem}

\usepackage{hyperref}
\usepackage{placeins}

\usepackage[preprint]{icml2026}

\usepackage{amsmath}
\usepackage{amssymb}
\usepackage{mathtools}
\usepackage{amsthm}
\usepackage{float}

\usepackage[capitalize,noabbrev]{cleveref}

\theoremstyle{plain}
\newtheorem{theorem}{Theorem}[section]
\newtheorem{proposition}[theorem]{Proposition}
\newtheorem{lemma}[theorem]{Lemma}

\theoremstyle{definition}
\newtheorem{definition}[theorem]{Definition}

\theoremstyle{remark}

\usepackage[textsize=tiny]{todonotes}

\icmltitlerunning{\textbf{SCAN}: \textbf{S}parse \textbf{C}ircuit \textbf{A}nchor Interpretable \textbf{N}euron for Lifelong Knowledge Editing}
\begin{document}

\twocolumn[
\icmltitle{\textbf{SCAN}: \underline{S}parse \underline{C}ircuit \underline{A}nchor Interpretable \underline{N}euron for Lifelong Knowledge Editing}



  \icmlsetsymbol{equal}{*}

  \begin{icmlauthorlist}
    \icmlauthor{Yuhuan Liu}{1,2}
    \icmlauthor{Haitian Zhong}{1,4}
    \icmlauthor{Xinyuan Xia}{3}
    \icmlauthor{Qiang Liu}{1}
    \icmlauthor{Shu Wu}{1}
    \icmlauthor{Liang Wang}{1}

  \end{icmlauthorlist}

  \icmlaffiliation{1}{New Laboratory of Pattern Recognition (NLPR), State Key Laboratory of Multimodal Artificial Intelligence Systems (MAIS), Institute of Automation, Chinese Academy of Sciences}
  \icmlaffiliation{2}{Cuiying Honors College, Lanzhou University}
  \icmlaffiliation{3}{Research Institute of Intelligent Complex Systems, Fudan University}
  \icmlaffiliation{4}{Zhongguancun Academy}

  \icmlcorrespondingauthor{Shu Wu}{shu.wu@nlpr.ia.ac.cn }


  \vskip 0.3in
]



\printAffiliationsAndNotice{}  

\begin{abstract}
Large Language Models (LLMs) often suffer from catastrophic forgetting and collapse during sequential knowledge editing. This vulnerability stems from the prevailing dense editing paradigm, which treats models as black boxes and relies on coarse-grained parameter interventions that inevitably disrupt preserved knowledge. To address this, we propose SCAN (a sparse editing framework based on \underline{S}parse \underline{C}ircuit \underline{A}nchored \underline{N}euron) which transforms editing into a mechanism-aware manipulation by constructing a knowledge circuit via Sparse Transcoders. Experiments on Gemma2, Qwen3, and Llama3.1 across CounterFact, ZsRE and WikiFactDiff demonstrate that SCAN achieves a superior performance, maintaining model integrity on benchmarks like MMLU and GSM8K even after 3,000 sequential edits, whereas other existing methods deteriorate progressively as editing accumulates, eventually resulting in model collapse.
\end{abstract}

\section{Introduction}

Large Language Models (LLMs) acquire extensive knowledge during pre-training which may become outdated. Given the high costs of re-training, model editing \cite{zhang-etal-2024-knowledge-editing,wang2024knowledge,wang-etal-2024-easyedit} has emerged as an efficient approach for updating specific factual knowledge without altering unrelated knowledge and general competency. For instance, an LLM asserting, ``The current U.S. President is Joe Biden,'' would need correction to, ``The current U.S. President is Donald Trump,'' in 2025. Existing editing techniques fall into two categories: parameter-modifying which directly alter the LLMs' weights, and parameter-preserving methods which introduce auxiliary components to steer model's output. Moreover, to accommodate evolving knowledge, lifelong model editing is proposed, which involves performing sequential edits \cite{gupta-etal-2024-model}.

Despite advances, current editing techniques face systemic limitations that are problematic in the lifelong setting. Chief among these is the issue of coarse editing granularity: methods often \textbf{intervene over large blocks of parameters} \cite{meng2022locating,meng2022mass}. However, the parameters truly essential to a specific fact occupy only a tiny fraction of these regions \cite{jiang2025neuron}. This coarse intervention directly causes the destruction of unrelated knowledge \cite{jiang2025neuron} due to the vector-based storage mechanism in MLP \cite{geva2021transformer} and polysemantic problems \cite{geva2022transformer}. In sequential editing, as these non-minimal intervention accumulate, they inevitably erode previously edited knowledge, triggering catastrophic forgetting and model collapse. Furthermore, the \textbf{lack of transparency} in knowledge pathways hinders reliable diagnosis and surgical refinement, making updates unreliable and uncontrollable \cite{hong2024interpretability,mazzia2024survey}.

To address these coupled challenges, we draw inspiration from the field of Mechanistic Interpretability \cite{bereska2024mechanistic,rai2024practical}, emphasizing \textbf{sparsity} as a promising way to solve them. We argue that sparsity enables finer grained editing through two reasons. Firstly, unlike dense paradigms that intervene in entire weight blocks, sparsity restricts updates to factual related parameter subsets, protecting unrelated knowledge. Secondly, sparsity offers a path toward neuron \textbf{monosemanticity} \cite{cunningham2023sparse,paulo2025transcoders}. Under this paradigm, each feature or neuron represents a single concept rather than multiple meanings \cite{bricken2023monosemanticity}, ensuring that editing a target concept does not propagate changes to unrelated concepts, thereby mitigating the polysemantic issues \cite{geva2022transformer} inherent in editing dense LLMs.

Based on this principle, this paper proposes a novel parameter-preserving editing framework SCAN. Our approach introduces a Sparse Transcoder \cite{dunefsky2024transcoders} which projects LLMs hidden states onto sparse features. We then construct an Attribution Graph by influence score between features and prune it to identify a knowledge circuit. This circuit serves as a roadmap to locate the essential and sparse feature nodes for edit instead of the whole transcoder features. By applying targeted steering to these sparse features, we implement the edit and propagate the refined changes back into the LLMs. Our main contributions are summarized as follows:

\begin{enumerate}
    \item We introduce sparsity as a core principle to resolve catastrophic forgetting and model collapse in sequential edit by restricting updates to fact-specific parameter subsets. This paradigm also fosters neuron monosemanticity, ensuring that each feature represents a discrete concept to mitigate the polysemantic interference inherent in dense models.
    \item We propose SCAN, a novel white-box sparse editing framework driven by Sparse Transcoders and Attribution Graphs. By identifying specific knowledge circuits, SCAN provides a robust and interpretable solution for the model editing field.
    \item We conduct a series of experiments to investigate sequential scalability, precise knowledge localization, and the semantic profiling of functional features across multiple LLM families. Our evaluations demonstrate that these sparse interventions successfully maintain model integrity and general capabilities even under long-term editing stress.
\end{enumerate}

\section{Preliminary}
\subsection{Lifelong Editing and Steer-based Method}
Model editing updates a model $f_{W}$ using a triple $e=(s, r, o \to o^*)$ where $s$ denotes the subject entity, $r$ the relation, $o$ the original factual object, and $o^*$ the desired target new object. In a lifelong setting, a Model Editor (ME) recursively produces $f_{W_t} = \text{ME}(f_{W_{t-1}}, x_t, y_t)$ to incorporate $n$ sequential updates without forgetting. Steer-based methods achieve this by modifying hidden states $h$ instead of weights \cite{zhong-etal-2025-react}. They typically utilize a Decision Mechanism (e.g., Euclidean distance to stored keys $\mathbb{K}_i$) to trigger a Perturbation Mechanism: $h' = h + \Delta h$ if an edit is required, and $h' = h$ otherwise \cite{hartvigsen2023aging,yu2024melo}.

\subsection{Editing Mechanism: MLPs as Key-Value Memory}
The Transformer MLP is conceptualized as a key-value memory \cite{geva2021transformer}, where the encoder weights $W_{\text{enc}}$ act as patterns (keys) and the decoder weights $W_{\text{dec}}$ act as stored knowledge (values). Physically, the input $h_{\text{pre}}$ is compared against $W_{\text{enc}}$ (keys) to produce an activation $a = \sigma(W_{\text{enc}} \cdot h_{\text{pre}})$, where each $a_j$ represents the response intensity that the input belongs to the $j$-th knowledge slot. The final output $v$ is a weighted retrieval of these values: $v = W_{\text{dec}} \cdot a = \sum a_j v_j$, which modifies the residual stream $h_{\text{post}} = h_{\text{pre}} + v$ \cite{elhage2021mathematical}. Traditional parameter-modifying methods (e.g., MEMIT, AlphaEdit) achieve editing by updating the decoder weights $W'_{\text{dec}} = W_{\text{dec}} + \Delta W_{\text{dec}}$ to encode new facts. Mathematically, this transformation is equivalent to injecting an additive perturbation $\Delta h = \Delta W_{\text{dec}} \cdot a$ into the residual stream, providing a unified view for both weight-based and steer-based editing.

\subsection{Sparse Transcoder and Monosemanticity}

The Sparse Transcoder is an auxiliary module that reformulates the MLP's memory into a highly sparse feature activation $z$ \cite{paulo2025transcoders} (compared with $a$ in MLP). It is trained to reconstruct the MLP output while enforcing sparsity via an $\ell_1$ penalty:
\[
\begin{aligned}
\mathcal{L}_{\text{Transcoder}} &=
\text{MSE}\big(\text{MLP}(h_{\text{pre}}), W_{\text{dec}}^{\text{tc}} \cdot z\big)
+ \lambda \lVert z \rVert_1, \\
z &= \text{ReLU}(W_{\text{enc}}^{\text{tc}} \cdot h_{\text{pre}})
\end{aligned}
\]

Unlike standard MLPs that suffer from polysemantic neurons, the Transcoder promotes monosemanticity \cite{bricken2023monosemanticity}.

\section{Edit with Sparse Circuit Anchored Neuron 
}

\begin{figure*}[t]
    \centering

    \includegraphics[width=0.98\linewidth, trim=10pt 2pt 10pt 2pt, clip]{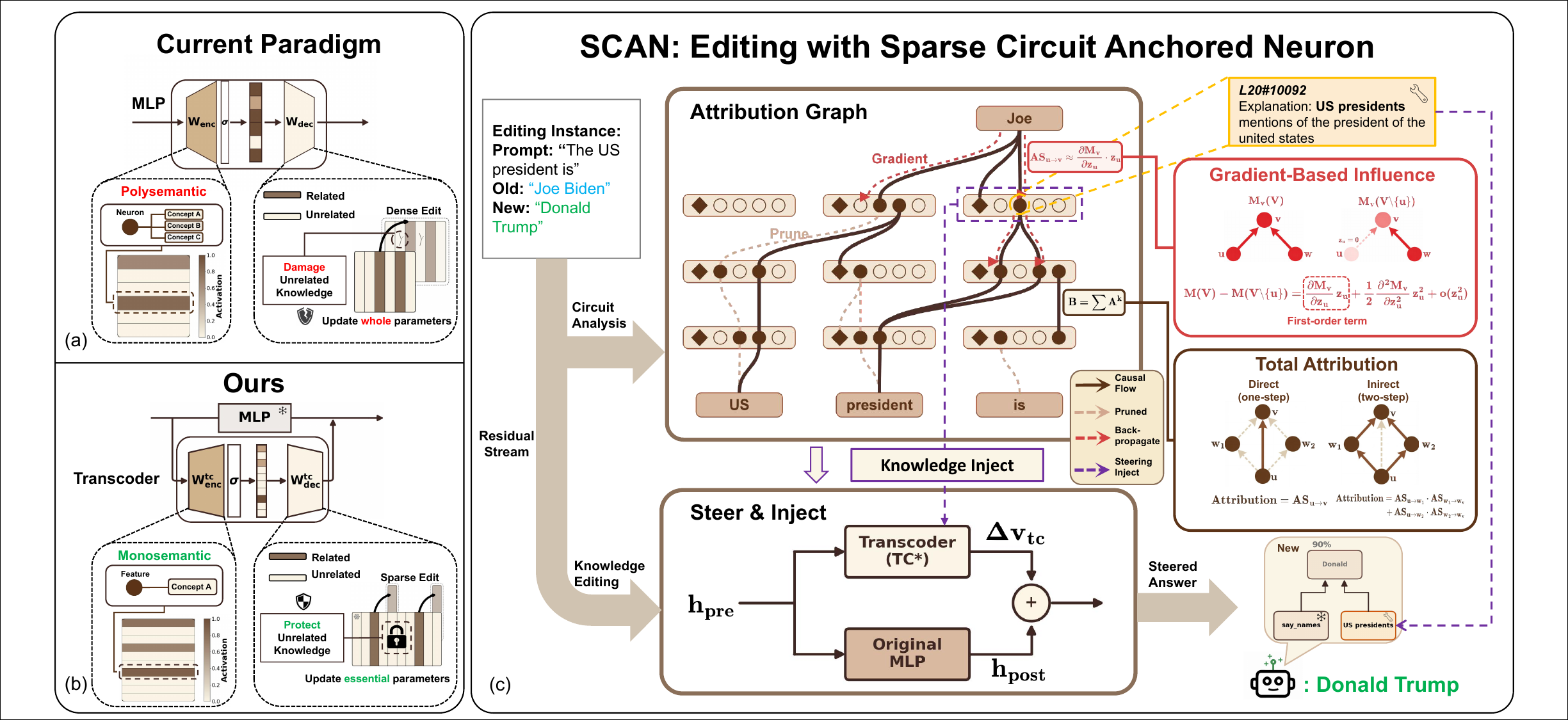}
    \caption{Comparison of current methods and ours. Current methods (a) modify the entire dense MLP weight matrix. Our approach (b) isolates factual features, editing knowledge-relevant vectors.}
    \label{fig:placeholder}
\end{figure*}

This section details SCAN. We define a single edit as $e = (s, r, o \to o^*)$. The method first constructs the knowledge circuit responsible for the original output $o$, based on the activated features in the Sparse Transcoder and attribution score between feature nodes. The essential features are then extracted to locate the knowledge and we edit the corresponding decoder weight vectors that enforce the target output $o^*$. Finally, the difference between the transcoder's original output and the edited output ($\Delta v_{\text{tc}}$) is injected into the LLMs as a ``steering vector'' at the exact circuit location. 

\subsection{Attribution Graph Construction}

The initial phase involves constructing an Attribution Graph for the old knowledge answer $o$, defined by the editing instance $e$. We execute a forward pass with the prompt to process it with the LLM. The hidden states ($h_{pre}$) before MLP from each layer and the token positions are then inputted into the corresponding transcoders to record the features with positive activation. A weighted complete graph is then constructed, formally defined as follows:

\begin{definition}[Initiation of Attribution Graph]
For an editing instance $e$, the corresponding Attribution Graph is initiated as a weighted complete graph: 
\[
G = (V, E, AS)
\]
where $V$ is a set of nodes representing all components potentially causal to the original output $o$ and is partitioned as: $V = V_{\text{embed}} \cup V_{\text{feature}} \cup V_{\text{error}} \cup V_{\text{logit}}$, here $V_{\text{embed}}$ denotes the token embedding node; $V_{\text{feature}}$ comprises the feature with positive activation $z_i$; $V_{\text{error}}$ contains nodes corresponding to the MLP reconstruction error; $V_{\text{logit}}$ denotes the logit node associated with the original object token $o$.
$E = \{(u, v) \mid u, v \in V, u \neq v\}$ connects every pair of nodes, and $AS : E \rightarrow \mathbb{R}$ assigns an attribution score representing the causal influence between nodes. In the initiated graph, all attribution scores are set to $0$.
\end{definition}
The visualization of this period can be seen in Appendix C.

\subsection{Gradient-Based Node Influence Computation}

With the complete Attribution Graph $G=K(V)$ established, we quantify the causal influence between the nodes. Suppose $u, v \in V$ are two nodes, where $u$ is located at layer $j$ and token position $m$, and $v$ is located at layer $i$ and token position $n$ $(i>j, n\geq m)$. Let $z_u$ and $z_v$ denote the corresponding activations respectively (for error and embedding nodes, the activations are set to 1). The attribution score $AS_{u \to v}$ is computed as the difference in the metric $M$ of node $v$ when the path through $u$ is corrupted while other paths are kept. Specifically, the influence of $u$ on $v$ is given by $AS_{u \to v} =  M_v(V)-M_v(V \setminus \{u\})$, here  $M_v(V)$ represents the value of the metric $M$ (e.g., absolute activation or activation relative to the mean; in our experiments, we use absolute activation which is $\text{ReLU}((f_{\text{enc}}^{v})^{\top} \cdot h_{\text{pre}}^{i})$) at node $v$ during the original forward pass. $M_v(V \setminus \{u\})$ represents the value of $M$ at node $v$ when the path through node $u$ is corrupted.

Direct computation of the attribution score leads to an immense computational cost. To address this, we use the first-order term of the Taylor expansion to approximate it, where removing node $u$ is equivalent to setting its activation $z_u$ to zero:
\[
AS_{u \to v} \approx \frac{\partial{M_v}}{\partial{z_u}} \cdot z_u
\]

To efficiently compute the gradient $\frac{\partial M_v}{\partial z_u}$, we utilize the Chain Rule across the Transformer layers, only requiring backpropagation to the output ( $h_{\text{post}}^{j} = h^{j}_{\text{pre}} + h^{j}_{\text{error}} + \sum_k z_k f^{k}_{\text{dec}}$) after the MLP. This results in a simple expression for the gradient: $\frac{\partial M_v}{\partial z_u} = \frac{\partial M_v}{\partial h_{\text{post}}^{j}} \frac{\partial h_{\text{post}}^{j}}{\partial z_u} = \frac{\partial M_v}{\partial h_{\text{post}}^{j}} f^{u}_{\text{dec}}$

We can further provide a mathematical interpretation of this process. In fact, the causal influence computation can be seen as evaluating the similarity between vectors via a dot product. By further expanding $\frac{\partial M_v}{\partial h_{\text{post}}^{j}} f^{u}_{\text{dec}}$, we get:
\begin{align*}
\frac{\partial M_v}{\partial h_{\text{post}}^{j}} f^{u}_{\text{dec}} 
&= \frac{\partial M_v}{\partial h_{\text{pre}}^{i}} \frac{\partial h_{\text{pre}}^{i}}{\partial h_{\text{post}}^{j}} f^{u}_{\text{dec}}\\& = (f_{\text{enc}}^{v})^{\top} \frac{\partial h_{\text{pre}}^{i}}{\partial h_{\text{post}}^{j}} f^{u}_{\text{dec}}
\end{align*}
This decomposed form is interpreted as the similarity between feature $u$ and feature $v$ in the pre-MLP space of the layer $i$. The vector $f^{u}_{\text{dec}}$, which encodes the value information stored by feature $u$ at layer $j$, is first transformed into the pre-MLP space of layer $i$ via the Jacobian matrix $\frac{\partial h_{\text{pre}}^{i}}{\partial h_{\text{post}}^{j}}$. This Jacobian matrix maps the decoder vector $f^u_{\text{dec}}$ from vector representation in the layer $j$ post-MLP space into its corresponding vector representation in the layer $i$ pre-MLP space. We formalize this relationship in the following proposition, which ensures that such a transformation across different spaces is well-defined (the proof is provided in Appendix B):
\begin{proposition}[Jacobian as the Optimal Direction-Preserving Linearization]
Let $X,Y\subset \mathbb{R}^n$ be two spaces and let
$
f: X \to Y
$
be a mapping such that $f(0)=0$ and $f$ is differentiable at every point $x_0\in X$ with Jacobian matrix $J_f(x_0)$ and $J_f$ is continuous and non-singular at $0$. Then, we have
\[
\left\|
\frac{f(x_0)}{\|f(x_0)\|}
-
\frac{J_f(x_0)x_0}{\|J_f(x_0)x_0\|}
\right\|
\;\xrightarrow[x_0\to 0]{}\; 0
\]
This implies that the direction of any vector transformed by $f$ is closely aligned with the direction induced by the Jacobian transformation.

\end{proposition}

The transformed vector is then compared using a dot product with the key vector $(f_{\text{enc}}^{v})^{\top}$. Thus, the attribution score measures the similarity between the transformed value vector from feature of $u$ and the required key vector of $v$, all within the $i$ layer pre-MLP space.

\subsection{Prune by Total (one-step and multi-step) Attribution}
To derive the sparse causal subgraph $G'$ from the dense Attribution Graph $G$, we employ a pruning strategy based on previous attribution scores ($AS_{u \to v}$), which measure the direct (one-step) causal effect of node $u$ on node $v$.

\begin{definition}[Direct (one-step) Attribution Matrix]
Let $A \in \mathbb{R}^{|V| \times |V|}$ denote the adjacency matrix of $G$, where the element in row $v$ and column $u$ is: 
\[
a_{v,u} = AS_{u \to v},
\] 
representing the one-step attribution from node $u$ to node $v$.  
Denote the $A$ written in column-block form as
 $(a_1, a_2, \dots, a_{|V|})$, or equivalently in row-block form as $(b_1^\top, b_2^\top, \dots, b_{|V|}^\top)^\top$
\end{definition}

\paragraph{Two-step Attribution.} Analogous to a full-derivative expansion:
\begin{theorem}[Full-derivative expansion]
Consider a function \(y = y(x_1, x_2, \dots, x_n)\), where each \(x_i\) is itself a function of \(z\), i.e., \(x_i = x_i(z)\). Then the derivative of \(y\) with respect to \(z\) can be expressed using partial derivatives as
\[
\frac{\partial y}{\partial z} = \sum_{i=1}^n \frac{\partial y}{\partial x_i} \frac{\partial x_i}{\partial z}.
\]
\end{theorem}
which implies that the change in \(y\) with respect to \(z\) can be decomposed into contributions from the changes in each intermediate variable \(x_i\), we begin with the simplest case of indirect influence: two-step attribution. This measures the effect of a node $u$ on a node $v$ that is mediated by a single intermediate node. We sum the influence over all paths of length two (i.e., through all possible intermediate nodes \(w_i\)) given by $\sum_{i} \text{AS}_{u \to w_i} \cdot \text{AS}_{w_i \to v}$, which in vector-matrix notation corresponds to:
\[
(A^2)_{v,u} = \sum_{i} a_{v,w_i} a_{w_i,u} = b_v^\top \cdot a_u
\]
Thus, the adjacency matrix representing all two-step attributions in the graph is precisely $A^2$.

\paragraph{Three-step and Total Attribution.} Following the same principle, three-step attribution is computed recursively. It quantifies the influence transmitted through two intermediate nodes. We can conceptualize this as the sum of products of the two-step attribution from the source node $u$ to an intermediate node $\hat{w}_i$, and one-step attribution from $\hat{w}_i$ to $v$. (i.e. $\sum_{i} \left( \sum_{k} \text{AS}_{u \to w_k} \cdot \text{AS}_{w_k \to \hat{w}_i} \right) \cdot \text{AS}_{\hat{w}_i \to v}
$). Similar to situation in two-step, the three-step adjacency matrix is $A^3$. 

The total adjacency matrix, denoted by \(B\), accumulating the attribution scores from all direct and indirect paths. It is defined as the sum of the adjacency matrix for all path lengths from one-step to infinity. This results in a matrix series:$B = A + A^2 + A^3 + \dots$. In fact, we have the following proposition to calculate it (The proof will be shown in Appendix B, with implementation details, including convergence and feasibility, provided in Appendix C):

\begin{proposition}[Closed-form Total Attribution Matrix]
Let $A$ be the adjacency matrix of one-step attribution scores with $\|A\| < 1$. Then the total Attribution Matrix $B$, which accumulates contributions from all paths of any length, admits the closed-form solution: 
\[
B = (I - A)^{-1} - I
\]
\end{proposition}

For any target node $v$, We first sort all nodes $u$ that have a path to $v$ in descending order of $B_{v,u}$ and normalize these scores. Edge pruning is then performed using a cumulative threshold $\tau$: starting from the highest-ranked node, scores are sequentially accumulated. Once the cumulative sum reaches $\tau$, all remaining edges are discarded for they are less influential. Nodes that do not attribute to any other node after edge pruning is removed. After pruning all edges and nodes, we get $G'$. 

\subsection{Sparse Edit and Knowledge Inject}

This final step uses the features identified in $G'$ that are active on the subject's last token for direct model editing and testing. During the editing phase, the transcoder decoder vectors $f^{u}_{\text{dec}}$ corresponding to the these features $u \in G’$ are selected and edited. The difference between before and after the edit ($\sum_u z_u\cdot(\hat{f}^{u}_{\text{dec}}-f^{u}_{\text{dec}})$) is calculated and injected into the LLMs. The optimization objective is set to minimize the negative log-probability of the target output, $\hat{W}^{tc}_{dec} = argmin_{W_{\text{dec}}^{\text{tc}}} \left( - \log P(o^*| x) \right)$. The edited decoders are then saved during sequential edit. In the testing phase, these identified features serve as the steering triggers. We use the Jaccard Similarity Score $J(V_A, V_B)$, where $V_A$ denotes the set of edited features associated with the target knowledge recorded by the transcoder during the editing phase, and $V_B$ denotes the set of selected features extracted from the attribution graph of the test prompt. If the similarity exceeds a predefined threshold, the edited features relevant to the corresponding knowledge are used for injection. The score, calculated as $J(V_A, V_B) = \frac{|V_A \cap V_B|}{|V_A \cup V_B|}$ quantifies the overlap between the selected feature sets.

\section{Experiments}

In this section, we conduct extensive experiments to evaluate the performance of our proposed method and address the following research questions:

\begin{itemize}
    \item \textbf{RQ1:} How does our method perform in sequential editing scenarios?
    \item \textbf{RQ2:} Where does the model encode factual knowledge within its parameters?
    \item \textbf{RQ3:} What is the semantic meaning of the identified features for editing? 
    \item \textbf{RQ4:} How does the model maintain its general capabilities after editing? 
\end{itemize}
\subsection{Experimental Setup}
We summarize the LLMs, baseline methods, Transcoders, datasets, and evaluation metrics used in our experiments. Further details are provided in Appendix A.

\textbf{LLMs \& Baseline Methods.} We conducted experiments using three widely adopted LLMs: \texttt{Gemma2-2B}, \texttt{Qwen3-8B}, and \texttt{Llama3.1-8B-Instruct}. For comparison, we evaluated our method against several editing baselines, including Fine-Tuning (FT) \cite{zhu2020modifying}, MEMIT \cite{meng2022mass}, RECT \cite{gu-etal-2024-model}, AlphaEdit \cite{ICLR2025_29c8c615}, GRACE \cite{hartvigsen2023aging}, and MELO \cite{yu2024melo}.

\textbf{Transcoder.} We adopt publicly available pretrained transcoder checkpoints without additional training. Specifically, we use the Gemma2-2B-transcoders and Qwen3-8B-transcoders \cite{dunefsky2024transcoders,circuit-tracer}, as well as the Llama3.1-8B-Instruct-transcoders \cite{zhao2025verifying}. These checkpoints are used as-is throughout all experiments.

\textbf{Datasets.} To evaluate the performance of knowledge editing, we employed three standard benchmarks: CounterFact \cite{meng2022locating}, ZsRE \cite{levy-etal-2017-zero}, and WikiFactDiff \cite{ammar-khodja-etal-2024-wikifactdiff-large}. Furthermore, to assess the general capabilities of the model post-edit, we tested the edited models on six general datasets, including MMLU \cite{hendryckstest2021}.

\textbf{Evaluation Metrics.} In line with prior research, we assess performance using three key metrics: \textbf{Rel} (Reliability, also known as Edit Success Rate), \textbf{Gen} (Generalization Success Rate), and \textbf{Loc} (Locality Success Rate). 

\subsection{How does our method perform in sequential editing scenarios? (RQ1)}

\begin{table*}[t]
\centering
\caption{Sequential editing task performance comparison of our method and other methods after 1000 edits. The \textbf{Avg} column is calculated using the Harmonic Mean: $\text{Avg} = 3 / (\text{Rel}^{-1} + \text{Gen}^{-1} + \text{Loc}^{-1})$, to illustrate balanced performance across metrics. \textbf{Bold} and \underline{underline} denote the best and second-best results per column, respectively.}
\label{tab:final_results}

\renewcommand{\arraystretch}{1.0}
\setlength{\tabcolsep}{10pt}

\newcommand{\lighttext}[1]{{\color{black!70}#1}}

\resizebox{\textwidth}{!}{%
\Large 
\begin{tabular}{cc cccc | cccc | cccc} 

\toprule[1.5pt]
\multirow{2}{*}{\textbf{Method}} & \multirow{2}{*}{\textbf{Model}} &
\multicolumn{4}{c}{\textbf{CounterFact}} &
\multicolumn{4}{c}{\textbf{ZsRE}} &
\multicolumn{4}{c}{\textbf{WikiFactDiff}} \\
\cmidrule(lr){3-6} \cmidrule(lr){7-10} \cmidrule(lr){11-14}
& & \textbf{Rel} $\uparrow$ & \textbf{Gen} $\uparrow$ & \textbf{Loc} $\uparrow$ & \textbf{Avg} $\uparrow$
& \textbf{Rel} $\uparrow$ & \textbf{Gen} $\uparrow$ & \textbf{Loc} $\uparrow$ & \textbf{Avg} $\uparrow$
& \textbf{Rel} $\uparrow$ & \textbf{Gen} $\uparrow$ & \textbf{Loc} $\uparrow$ & \textbf{Avg} $\uparrow$ \\
\midrule[1pt]

FT & \multirow{7}{*}{\rotatebox{90}{Gemma2-2B}} & \lighttext{44.28} & \lighttext{11.10} & \lighttext{11.77} & \lighttext{16.52}
& \lighttext{66.14} & \lighttext{57.20} & \lighttext{64.67} & \lighttext{62.43}
& \lighttext{73.07} & \lighttext{\underline{69.53}} & \lighttext{41.24} & \lighttext{56.97} \\
RECT & & \lighttext{7.23} & \lighttext{1.63} & \lighttext{20.90} & \lighttext{3.86}
& \lighttext{15.32} & \lighttext{11.58} & \lighttext{18.38} & \lighttext{14.56}
& \lighttext{35.78} & \lighttext{31.50} & \lighttext{31.64} & \lighttext{32.97} \\
AlphaEdit & & \lighttext{55.38} & \lighttext{18.35} & \lighttext{32.63} & \lighttext{28.91}
& \lighttext{73.88} & \lighttext{\underline{60.22}} & \lighttext{47.71} & \lighttext{59.11}
& \lighttext{76.92} & \lighttext{67.26} & \lighttext{55.73} & \lighttext{65.88} \\
MEMIT & & \lighttext{5.00} & \lighttext{2.80} & \lighttext{4.40} & \lighttext{3.77}
& \lighttext{10.91} & \lighttext{9.16} & \lighttext{7.61} & \lighttext{9.07}
& \lighttext{5.09} & \lighttext{4.68} & \lighttext{6.30} & \lighttext{5.32} \\
GRACE & & \textbf{100} & \lighttext{0.37} & \textbf{99.80} & \lighttext{1.10}
& \lighttext{\underline{98.80}} & \lighttext{23.56} & \textbf{100} & \lighttext{46.55}
& \lighttext{\underline{98.30}} & \lighttext{51.20} & \textbf{99.07} & \lighttext{75.50} \\
MELO & & \lighttext{69.92} & \lighttext{\underline{42.02}} & \lighttext{45.57} & \lighttext{\underline{49.97}}
& \lighttext{69.05} & \lighttext{56.11} & \lighttext{92.28} & \lighttext{\underline{69.43}}
& \lighttext{78.55} & \lighttext{68.82} & \lighttext{\underline{95.15}} & \lighttext{\underline{79.37}} \\
\rowcolor{red!10}
\textbf{SCAN (Ours)} & & \textbf{100} & \textbf{89.28} & \lighttext{\underline{91.97}} & \textbf{93.53}
& \textbf{100} & \textbf{97.87} & \textbf{100} & \textbf{99.28}
& \textbf{100} & \textbf{93.29} & \lighttext{90.46} & \textbf{94.43} \\
\specialrule{1.5pt}{1.5pt}{0.8pt}
\specialrule{1.5pt}{0.8pt}{1.5pt}
FT & \multirow{7}{*}{\rotatebox{90}{Qwen3-8B}} & \lighttext{6.15} & \lighttext{3.10} & \lighttext{5.05} & \lighttext{4.39}
& \lighttext{20.70} & \lighttext{20.21} & \lighttext{18.46} & \lighttext{19.79}
& \lighttext{26.11} & \lighttext{25.12} & \lighttext{26.60} & \lighttext{25.94} \\
RECT & & \lighttext{42.60} & \lighttext{27.75} & \lighttext{2.80} & \lighttext{7.23}
& \lighttext{30.44} & \lighttext{28.84} & \lighttext{11.16} & \lighttext{18.39}
& \lighttext{11.97} & \lighttext{9.99} & \lighttext{0.81} & \lighttext{2.18} \\
AlphaEdit & & \lighttext{91.90} & \lighttext{25.90} & \lighttext{74.50} & \lighttext{44.71}
& \lighttext{96.77} & \lighttext{\underline{77.35}} & \lighttext{87.45} & \lighttext{\underline{86.96}}
& \lighttext{70.41} & \lighttext{60.93} & \lighttext{68.50} & \lighttext{66.39} \\
MEMIT & & \lighttext{9.30} & \lighttext{5.40} & \lighttext{0.10} & \lighttext{0.29}
& \lighttext{36.50} & \lighttext{32.19} & \lighttext{20.37} & \lighttext{28.02}
& \lighttext{1.18} & \lighttext{0.74} & \lighttext{6.18} & \lighttext{1.64} \\
GRACE & & \textbf{100} & \lighttext{0.65} & \textbf{99.98} & \lighttext{1.94}
& \textbf{100} & \lighttext{27.40} & \textbf{100} & \lighttext{51.60}
& \lighttext{\underline{99.85}} & \lighttext{46.12} & \textbf{98.84} & \lighttext{70.34} \\
MELO & & \lighttext{88.00} & \lighttext{\underline{32.45}} & \lighttext{64.55} & \lighttext{\underline{50.94}}
& \lighttext{79.06} & \lighttext{66.74} & \lighttext{99.46} & \lighttext{80.79}
& \lighttext{72.40} & \lighttext{\underline{61.00}} & \lighttext{\underline{97.10}} & \lighttext{\underline{75.31}} \\
\rowcolor{red!10}
\textbf{SCAN (Ours)} & & \textbf{100} & \textbf{98.25} & \lighttext{\underline{92.95}} & \textbf{96.97}
& \textbf{100} & \textbf{95.36} & \textbf{100} & \textbf{98.43}
& \textbf{100} & \textbf{89.29} & \lighttext{90.30} & \textbf{92.97} \\
\specialrule{1.5pt}{1.5pt}{0.8pt}
\specialrule{1.5pt}{0.8pt}{1.5pt}
FT & \multirow{7}{*}{\rotatebox{90}{Llama3.1-8B}} & \lighttext{37.20} & \lighttext{12.60} & \lighttext{1.95} & \lighttext{4.79}
& \lighttext{56.59} & \lighttext{47.04} & \lighttext{4.97} & \lighttext{11.38}
& \lighttext{73.07} & \lighttext{68.28} & \lighttext{19.42} & \lighttext{32.23} \\
RECT & & \lighttext{10.35} & \lighttext{8.10} & \lighttext{0.40} & \lighttext{1.16}
& \lighttext{6.18} & \lighttext{5.45} & \lighttext{2.63} & \lighttext{4.18}
& \lighttext{2.61} & \lighttext{1.93} & \lighttext{1.48} & \lighttext{1.95} \\
AlphaEdit & & \lighttext{97.65} & \lighttext{39.55} & \lighttext{45.85} & \lighttext{56.68}
& \lighttext{94.48} & \lighttext{\underline{82.41}} & \lighttext{78.14} & \lighttext{\underline{84.91}}
& \lighttext{92.87} & \lighttext{\underline{86.88}} & \lighttext{64.06} & \lighttext{78.07}\\
MEMIT & & \lighttext{0.00} & \lighttext{0.00} & \lighttext{0.60} & \lighttext{0.00}
& \lighttext{1.19} & \lighttext{0.82} & \lighttext{3.69} & \lighttext{1.41}
& \lighttext{0.04} & \lighttext{0.04} & \lighttext{0.18} & \lighttext{0.06} \\
GRACE & & \textbf{100} & \lighttext{1.00} & \textbf{99.80} & \lighttext{2.94}
& \lighttext{\underline{99.85}} & \lighttext{27.35} & \textbf{100} & \lighttext{51.44}
& \lighttext{\underline{99.84}} & \lighttext{59.13} & \textbf{99.41} & \lighttext{79.17} \\
MELO & & \lighttext{90.05} & \lighttext{\underline{59.70}} & \lighttext{48.05} & \lighttext{\underline{63.67}}
& \lighttext{88.48} & \lighttext{70.19} & \lighttext{88.61} & \lighttext{81.83}
& \lighttext{84.33} & \lighttext{73.63} & \lighttext{\underline{92.24}} & \lighttext{\underline{83.11}} \\
\rowcolor{red!10}
\textbf{SCAN (Ours)} & & \textbf{100} & \textbf{86.70} & \lighttext{\underline{95.10}} & \textbf{93.58}
& \textbf{100} & \textbf{91.07} & \lighttext{\underline{99.88}} & \textbf{96.84}
& \textbf{100} & \textbf{87.45} & \lighttext{89.05} & \textbf{92.07} \\
\bottomrule[1.5pt]
\end{tabular}
}
\end{table*}

To assess our method, we evaluate across 1,000 sequential edits. The batch size for all batch editing methods is set to 100. As demonstrated in Table \ref{tab:final_results}, our method outperforms all baselines almost in all metrics. Notably, SCAN provides a balanced solution by sustaining near-perfect scores across all three dimensions compared with others.

\subsection{Where does the model encode factual knowledge within its parameters? (RQ2)}
\begin{figure}[t]
    \centering
    \subfigure[]{
        \includegraphics[width=0.46\columnwidth,
        trim=5pt 0pt 5pt 0pt,
        clip]{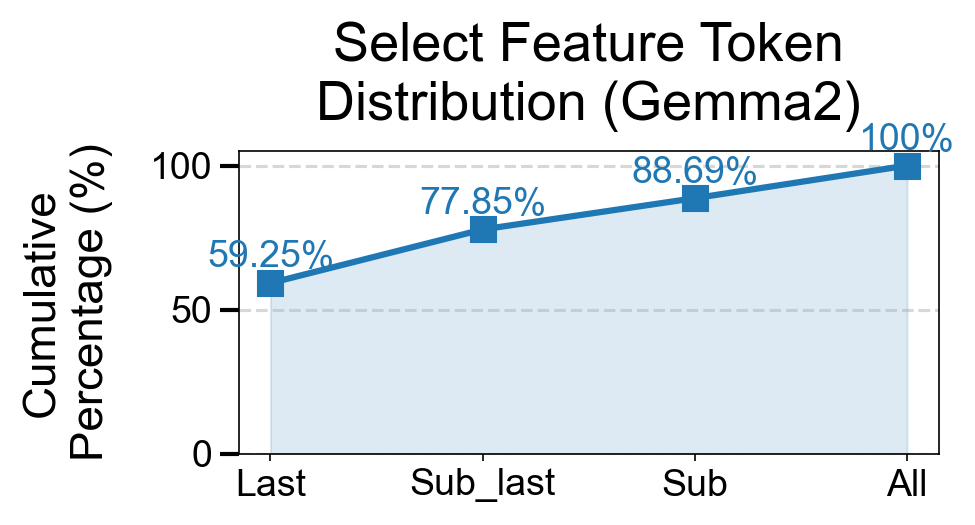}
        \label{fig:dist_gemma}
    }
\hspace{-0.1em}
    \subfigure[]{
        \includegraphics[width=0.46\columnwidth,
        trim=5pt 0pt 5pt 0pt,
        clip]{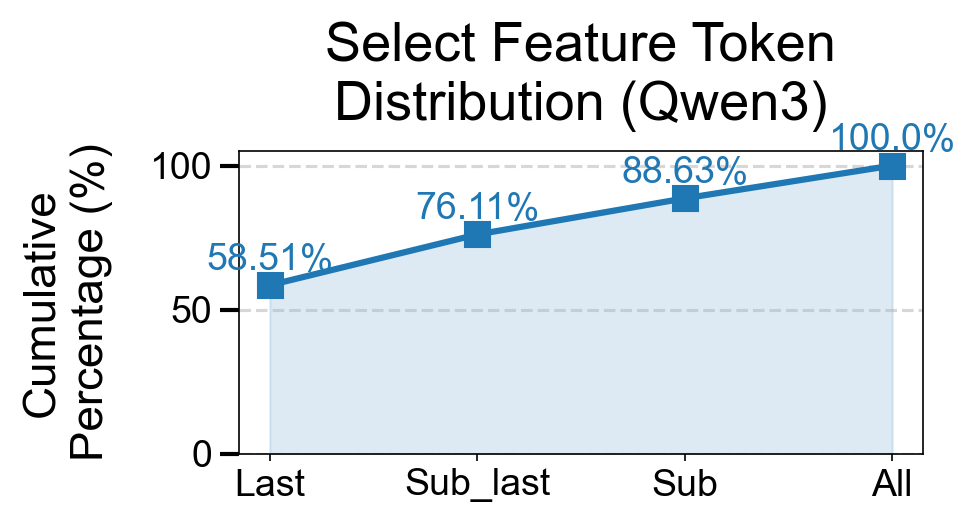}
        \label{fig:dist_qwen}
    }
    \caption{Cumulative proportion of selected feature across different token positions. (a) and (b) represent the distribution for Gemma2-2B and Qwen3-8B on CounterFact dataset, respectively.}
    \label{fig:token_distribution_combined}
\end{figure}

To investigate the internal localization of factual knowledge, we perform an analysis on the first 1,000 cases of the CounterFact dataset using Gemma2-2B and Qwen3-8B. Unlike prior work that performs localization only at the layer level \cite{meng2022locating,zhang2024locate}, our approach enables finer-grained localization, identifying specific subcomponents. We analyze the characteristics of key nodes from both the positional and layer-wise dimensions. The example Attribution Graph can be see in Appendix C.

\textbf{Positional Localization.} As illustrated in Figure \ref{fig:token_distribution_combined}, the attribution of factual knowledge is highly concentrated at last token contributing approximately half of the total nodes. When combined with the last token of the subject, these two positions cumulatively account for over 75\% of the entire Attribution Graph. In contrast, nodes at other positions, such as those within the prompt or earlier parts of the subject, exhibit lower contributions. This indicates that the subject's boundary and the final token are the dominant positions influencing the model's inference.

\begin{figure}[t]
    \centering
    \subfigure[]{
        \includegraphics[width=0.46\columnwidth,
        trim=5pt 0pt 5pt 0pt,
        clip]{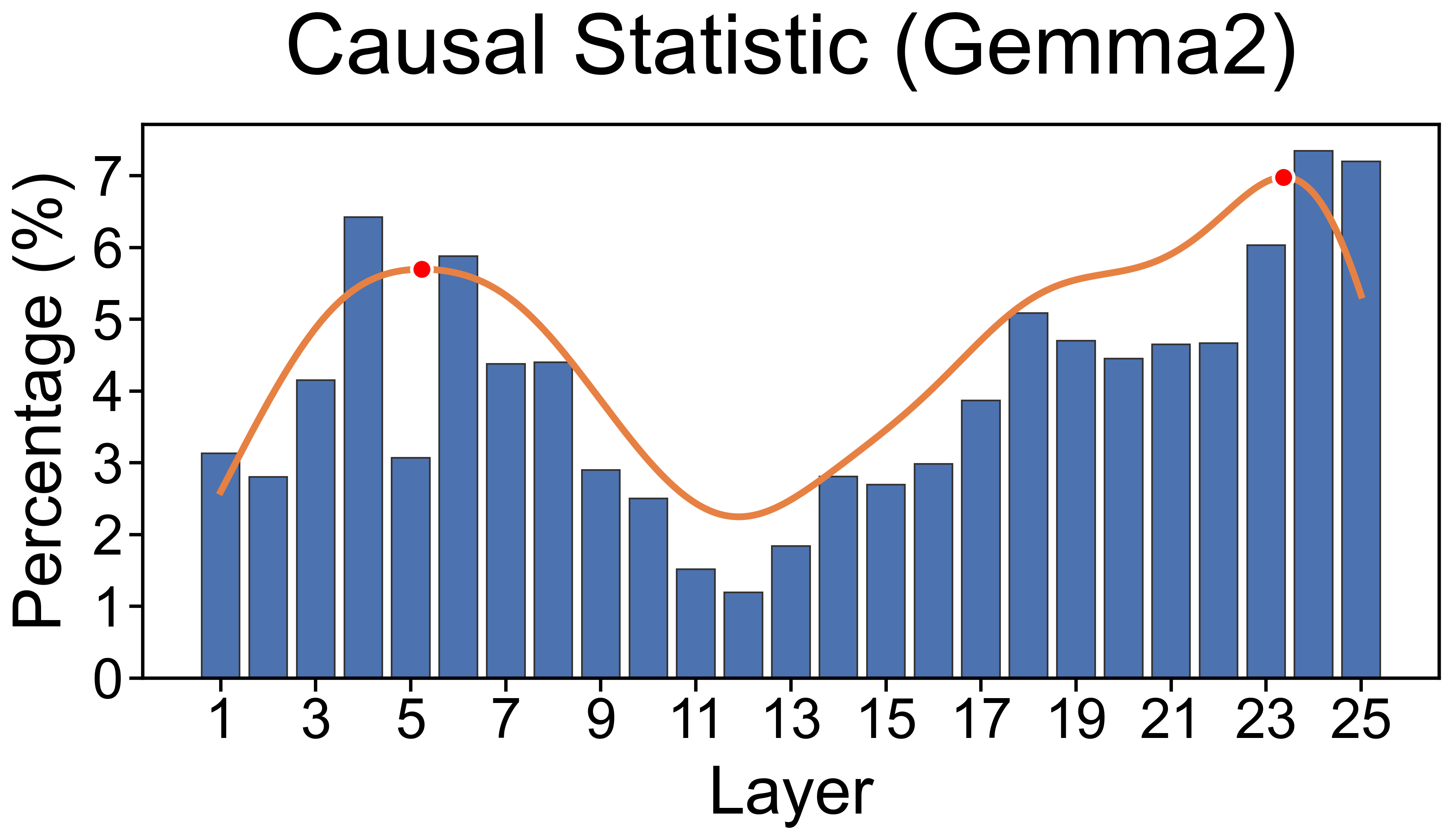}
        \label{fig:layer_gemma}
    }
\hspace{-0.1em}
    \subfigure[]{
        \includegraphics[width=0.46\columnwidth,
        trim=5pt 0pt 5pt 0pt,
        clip]{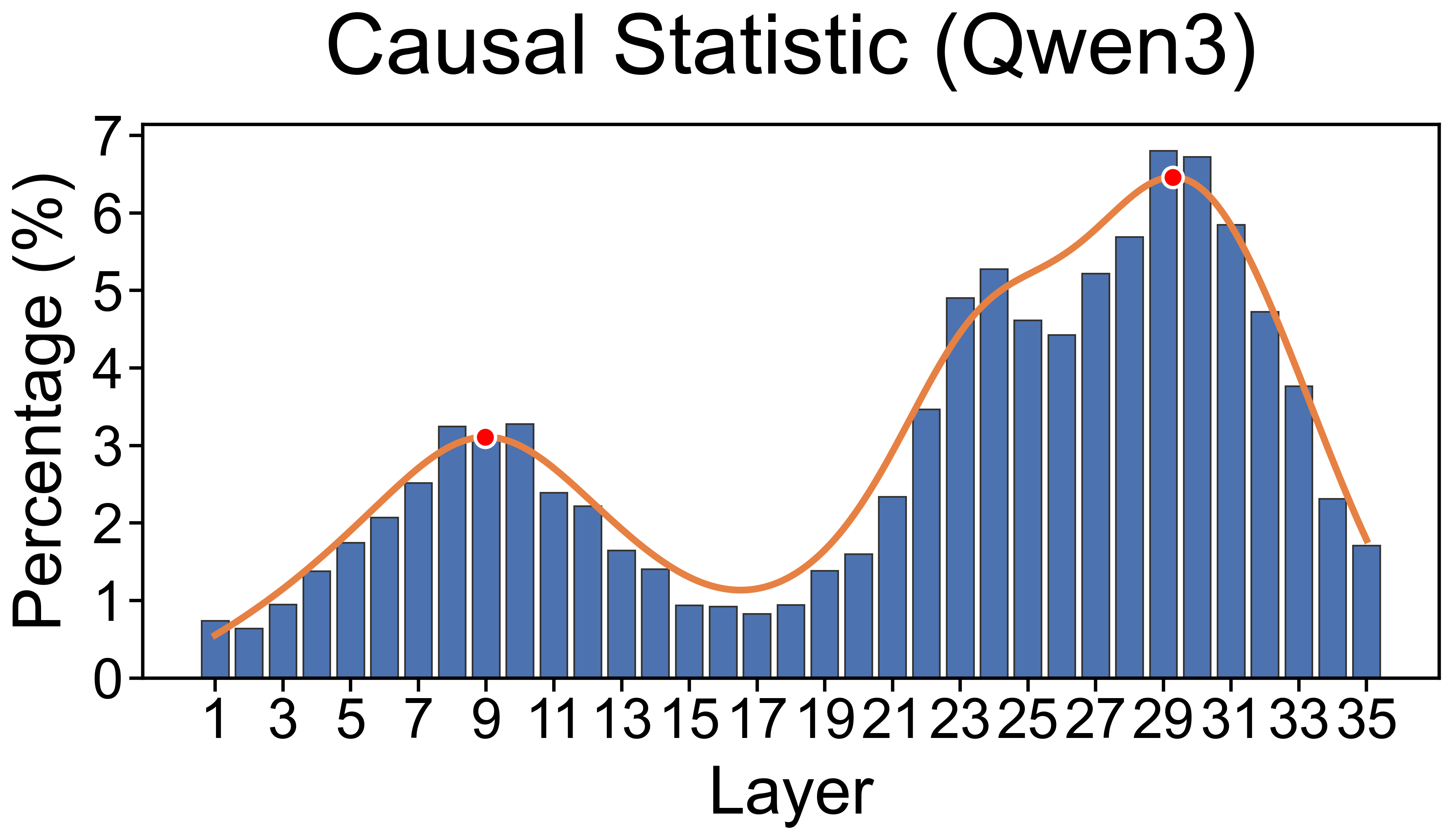}
        \label{fig:layer_qwen}
    }
    \caption{Distribution of selected feature across layers. Both models exhibit a characteristic dual-peak pattern, indicating functional localization in shallow and middle-to-deep layers.}
    \vspace{0em}
    \label{fig:layer_distribution_combined}
\end{figure}

\begin{figure}[t]
    \centering
    \subfigure[]{
        \includegraphics[width=0.40\textwidth, trim=3pt 3pt 3pt 3pt, clip]{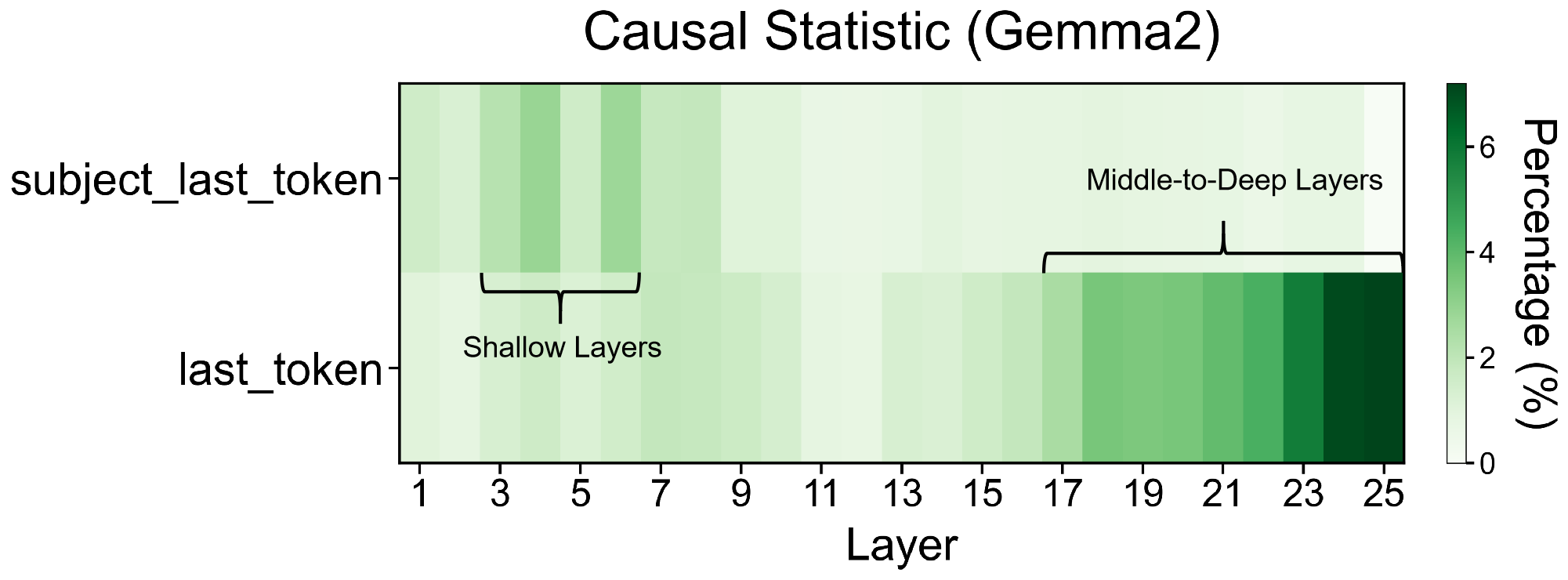}
        \label{fig:heatmap_gemma}
    }
    \subfigure[]{
        \includegraphics[width=0.40\textwidth, trim=3pt 3pt 3pt 3pt, clip]{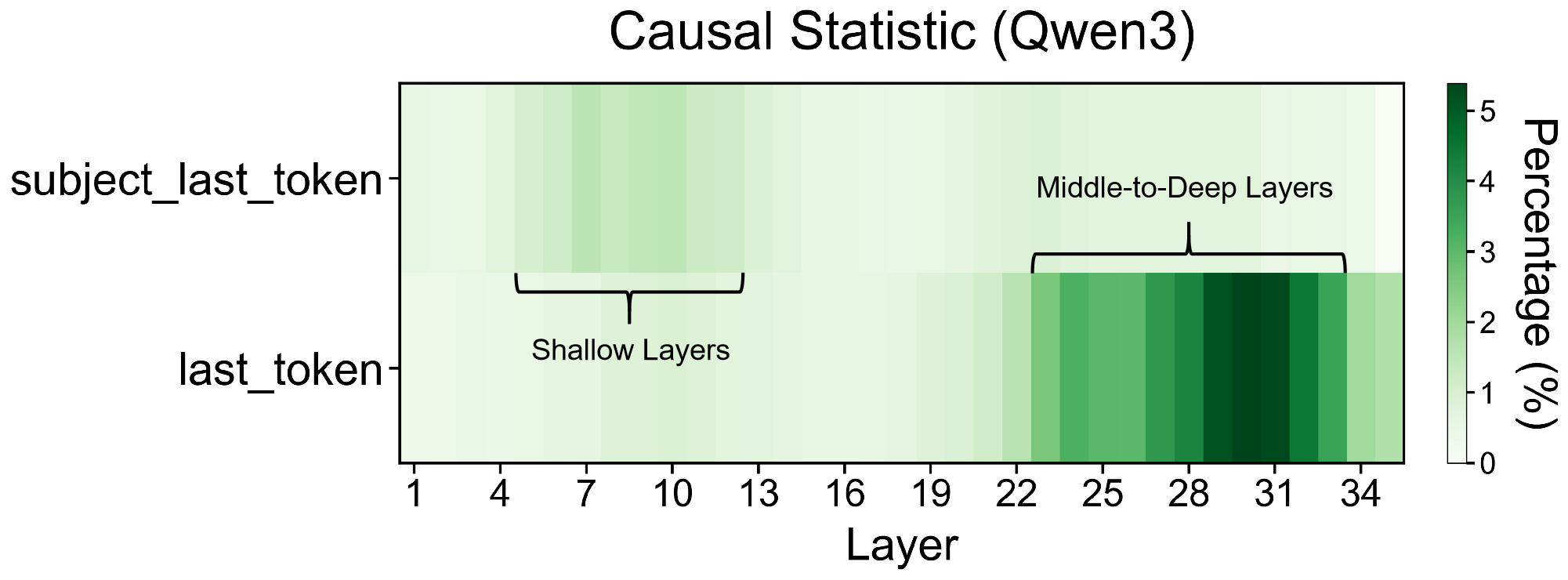}
        \label{fig:heatmap_qwen}
    }
    \caption{Heatmap of selected feature distribution across layers at special token position. The dark regions indicate that the early-layer peaks in Figure \ref{fig:layer_distribution_combined} align with the subject tokens, while the later-layer peaks correspond to the last token position on both models.}
    \vspace{0em}
    \label{fig:heatmaps_combined}
\end{figure}
\textbf{Layer-wise Distribution and Functional Specialization.} From the perspective of depth, the percentage of attribution nodes exhibits a distinct dual-peak pattern across layers, as shown in Figure \ref{fig:layer_distribution_combined}. To further decouple the functions of these peaks, we analyze the distribution of features activated at the subject last token and the last token separately using a heatmap (Figure \ref{fig:heatmaps_combined}). Our analysis reveals that the two peaks observed in the global distribution correspond to the deeper concentrations in the heatmap:
\begin{itemize}[topsep=0pt, itemsep=0pt, parsep=0pt]
    \item \textbf{Shallow Layers:} The peak here primarily corresponds to the \textit{subject last token}, where MLP layers focus on representing and stabilizing the subject's identity.
    \item \textbf{Middle-to-Deep Layers:} The second peak is dominated by the \textit{last token} position, where MLPs integrate the relation information with the subject's representation to extract the target answer.
\end{itemize}

In summary, these findings suggest a ``Subject-to-Answer'' pipeline: shallow layers encode the subject's information, while middle-to-deep layers leverage the relation to finalize the factual retrieval at the terminal token.

\subsection{What is the semantic meaning of the identified features for editing? (RQ3)}
\begin{table}[ht]
    \centering
    \caption{Comparison of feature discriminability across different selection methods using Qwen3-8B and Gemma2-2B on the CF dataset.}
    \label{tab:discrimination}

    \setlength{\heavyrulewidth}{0.12em} 
    \setlength{\lightrulewidth}{0.1em}  
    \setlength{\cmidrulewidth}{0.08em}  
    \setlength{\arrayrulewidth}{0.5pt}

    \resizebox{\columnwidth}{!}{%

        \renewcommand{\arraystretch}{1.2} 
        \setlength{\tabcolsep}{2.5pt}    
        \fontsize{9.5pt}{11pt}\selectfont

        \begin{tabular}{c cccc | cccc} 
            \toprule
            \multirow{2}{*}{\textbf{Method}} & \multicolumn{4}{c|}{\textbf{Rel-Gen}} & \multicolumn{4}{c}{\textbf{Rel-Loc}} \\
            \cmidrule(lr){2-5} \cmidrule(lr){6-9}
            & J-scr & J-acc & F1-scr & F1-acc & J-scr & J-acc & F1-scr & F1-acc \\
            \midrule
            \multicolumn{9}{c}{Qwen3-8B} \\
            \midrule
            All      & 0.258 & 0.421 & 0.394 & 0.423 & 0.297 & 0.409 & 0.445 & 0.394 \\
            Sub      & 0.594 & 0.992 & 0.731 & 0.994 & 0.162 & 0.931 & 0.272 & 0.928 \\
            No\_last & 0.346 & 0.689 & 0.496 & 0.694 & 0.151 & 0.880 & 0.249 & 0.877 \\
            Both     & 0.574 & 0.978 & 0.714 & 0.979 & 0.094 & 0.955 & 0.161 & 0.954 \\
            \midrule
            \multicolumn{9}{c}{Gemma2-2B} \\
            \midrule
            All      & 0.251 & 0.414 & 0.383 & 0.424 & 0.300 & 0.404 & 0.449 & 0.389 \\
            Sub      & 0.509 & 0.938 & 0.659 & 0.939 & 0.153 & 0.921 & 0.258 & 0.917 \\
            No\_last & 0.312 & 0.573 & 0.454 & 0.583 & 0.160 & 0.861 & 0.265 & 0.854 \\
            Both     & 0.491 & 0.926 & 0.642 & 0.927 & 0.118 & 0.938 & 0.203 & 0.937 \\
            \bottomrule
        \end{tabular}%
    }
\end{table}

\begin{figure*}[t]
    \centering

    \makebox[0.44\textwidth][c]{\small \textbf{Feature \#13366 at Layer 19}}
    \hfill
    \makebox[0.44\textwidth][c]{\small \textbf{Feature \#410 at Layer 24}}

    \subfigure[Reliability]{
    \makebox[0.42\textwidth][c]{%
        \adjustbox{trim=0 {\dimexpr0.3\height\relax} 0 {\dimexpr0.3\height\relax},clip}{%
            \includegraphics[width=0.48\textwidth]{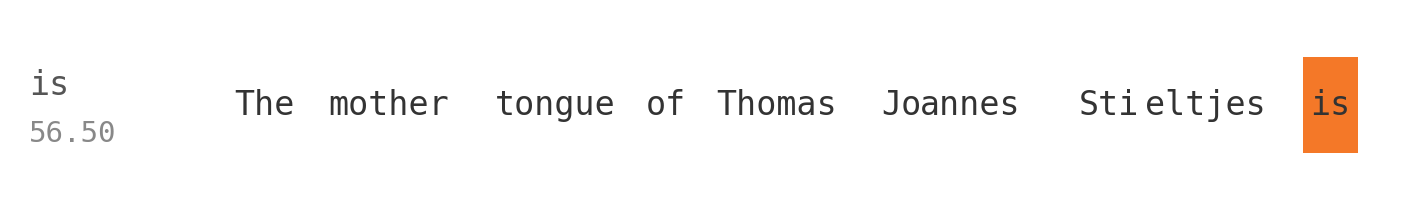}
        }
    }}
    \hfill
    \subfigure[Reliability]{
    \makebox[0.42\textwidth][c]{%
        \adjustbox{trim=0 {\dimexpr0.3\height\relax} 0 {\dimexpr0.3\height\relax},clip}{%
            \includegraphics[width=0.48\textwidth]{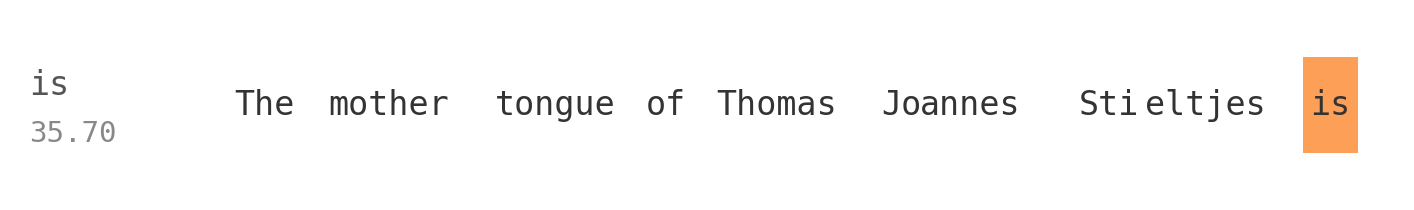}
        }
    }}

    \subfigure[Generality]{
    \makebox[0.42\textwidth][c]{%
        \adjustbox{trim=0 {\dimexpr0.3\height\relax} 0 {\dimexpr0.3\height\relax},clip}{%
            \includegraphics[width=0.48\textwidth]{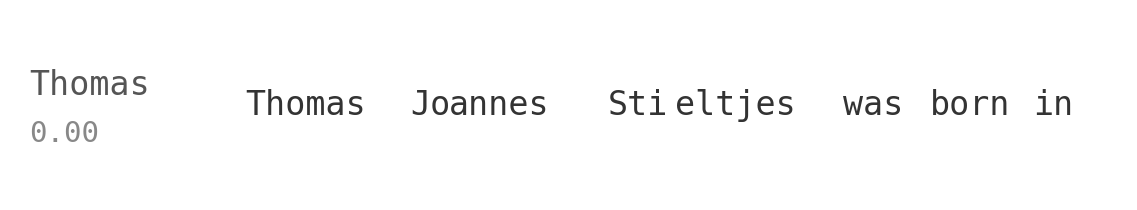}
        }
    }}
    \hfill
    \subfigure[Generality]{
    \makebox[0.42\textwidth][c]{%
        \adjustbox{trim=0 {\dimexpr0.3\height\relax} 0 {\dimexpr0.3\height\relax},clip}{%
            \includegraphics[width=0.48\textwidth]{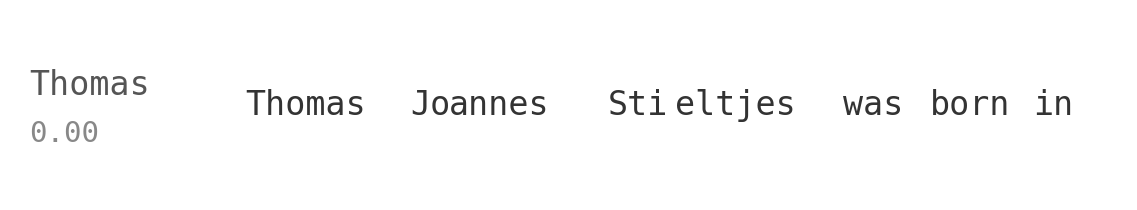}
        }
    }}

    \subfigure[Locality]{
    \makebox[0.42\textwidth][c]{%
        \adjustbox{trim=0 {\dimexpr0.3\height\relax} 0 {\dimexpr0.3\height\relax},clip}{%
            \includegraphics[width=0.48\textwidth]{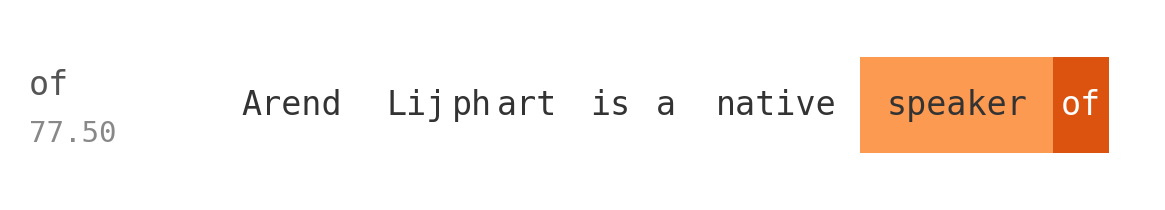}
        }
    }}
    \hfill
    \subfigure[Locality]{
    \makebox[0.42\textwidth][c]{%
        \adjustbox{trim=0 {\dimexpr0.3\height\relax} 0 {\dimexpr0.3\height\relax},clip}{%
            \includegraphics[width=0.48\textwidth]{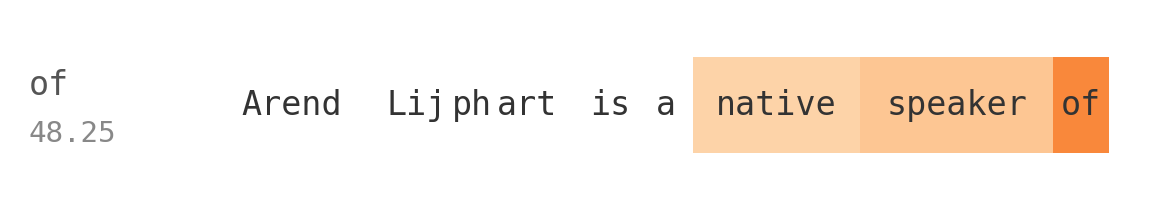}
        }
    }}

    \caption{Activation visualization of identified features on the specific prompts. 
    The left column shows Feature \#13366 at Layer 19, and the right column shows Feature \#410 at Layer 24. 
    Darker colors indicate higher activation values.}
    \label{fig:feature_activation_matrix_v2}
\end{figure*}
Not all nodes captured by Attribution Graph are suitable. These nodes often represent a mixture of \textit{edit-specific} features (tied to the unique fact) and \textit{general-purpose} features (common semantic categories). Indiscriminately editing the latter would lead to overfitting. To illustrate this, we examine a reliability prompt ``\textit{The mother tongue of Thomas Joannes Stieltjes is}'' and a locality prompt ``\textit{Arend Lijphart is a native speaker of}'' (both targeting the answer \textit{Dutch}). Our attribution analysis (see Figure \ref{fig:feature_activation_matrix_v2}) reveals that both prompts activate feature \#13366 at layer 19 and feature \#410 at layer 24. By projecting their decoder vectors onto the unembedding matrix, we extract the top-5 tokens with the highest logits:

\begin{itemize}[topsep=0pt, itemsep=0pt, parsep=0pt]
    \item \textbf{Layer 19, \#13366:} Spanish, English, Arabic, bahasa, and Hindi.
    \item \textbf{Layer 24, \#410:} Korean, Japanese, Indonesian, Vietnamese and Russian.
\end{itemize}
These nodes clearly represent the general concept of ``language'' rather than the specific fact being edited. Crucially, these general features do not activate at any position within the subject of the sentence, nor are they triggered by rephrased prompts such as ``\textit{Thomas Joannes Stieltjes was born in}''. This observation suggests that features relevant to the specific factual update are uniquely concentrated within the subject's boundary. More case study and visualization can be found in Appendix E.

Based on these insights, we design three filtering rules to extract specialized features by constraining their activation positions: (\textit{i}) Subject-Anchored Features (Sub): activated at the subject last token; (\textit{ii}) Terminal-Exclusionary Features (No\_last): not activated at the last token; and (\textit{iii}) Both: meeting both criteria. We evaluate the discriminative power of these rules on the first 1,000 cases of CounterFact. As shown in Table \ref{tab:discrimination}, we measure the average feature overlap between prompts using Jaccard and F1 scores. We record the average score and define discriminative accuracy based on thresholds (0.25 for Jaccard, 0.4 for F1, roughly half overlap between the sets). Our findings show that features activated at the subject last token (Sub) provide the highest discriminative scores. This confirms that the subject's terminal position is the primary locus for edit-specific factual information.

\subsection{How does the model maintain its general capabilities after editing? (RQ4)}

The experimental results across six diverse benchmarks—including ARC \cite{allenai:arc}, CommonsenseQA \cite{talmor-etal-2019-commonsenseqa}, GSM8K \cite{cobbe2021gsm8k}, MMLU \cite{hendryckstest2021}, OpenBookQA \cite{mihaylov-etal-2018-suit}, and SciQ \cite{welbl-etal-2017-crowdsourcing}. The experiments demonstrate that our method exhibits robustness in preserving model's general capabilities during editing process. As edit number scales up to 3,000, our approach maintains pre-existing ability. In contrast, other methods like RECT and MEMIT suffer from a collapse as the editing number increases, with their accuracy plummeting toward zero after 2,000 edits. Details are shown in Figure \ref{fig:downstream_benchmarks}.

\begin{figure*}[t]
    \centering
    \subfigure[]{
        \includegraphics[width=0.29\textwidth,
        trim=15pt 0pt 15pt 0pt,
        clip]{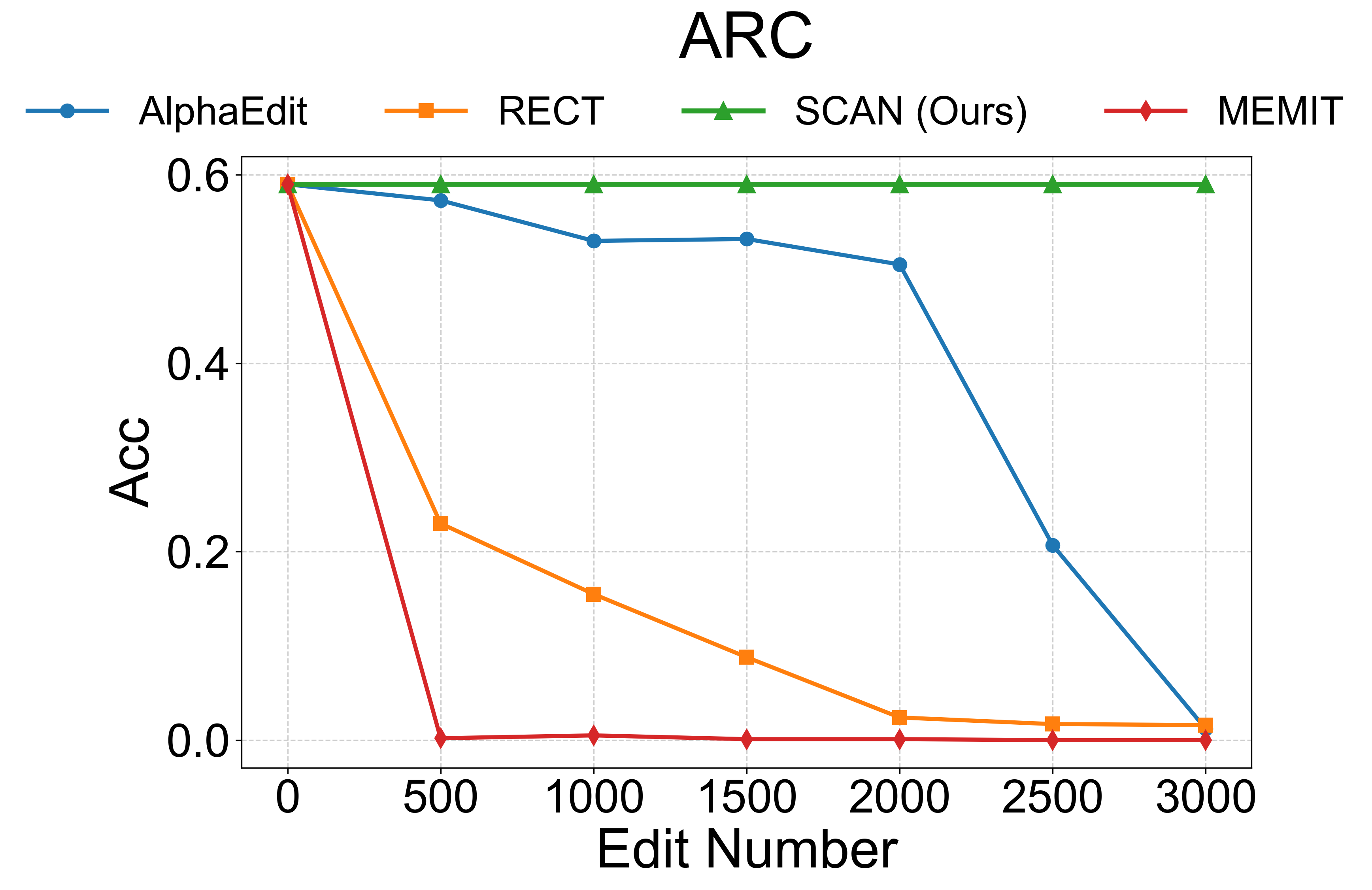}
        \label{fig:arc}
    }
\hspace{0.5em}
    \subfigure[]{
        \includegraphics[width=0.29\textwidth,
        trim=15pt 0pt 15pt 0pt,
        clip]{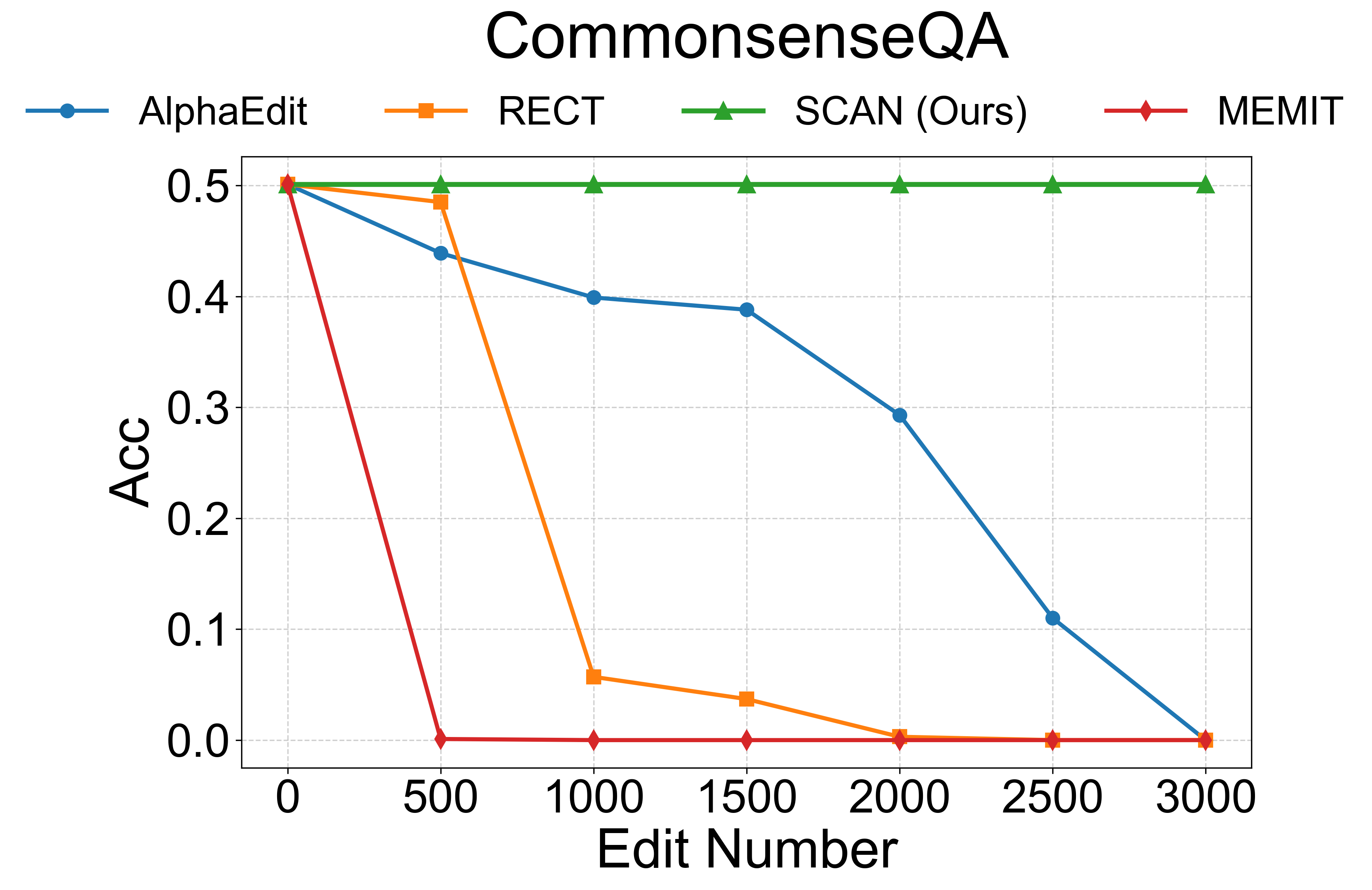}
        \label{fig:cqa}
    }
\hspace{0.5em}
    \subfigure[]{
        \includegraphics[width=0.29\textwidth,
        trim=15pt 0pt 15pt 0pt,
        clip]{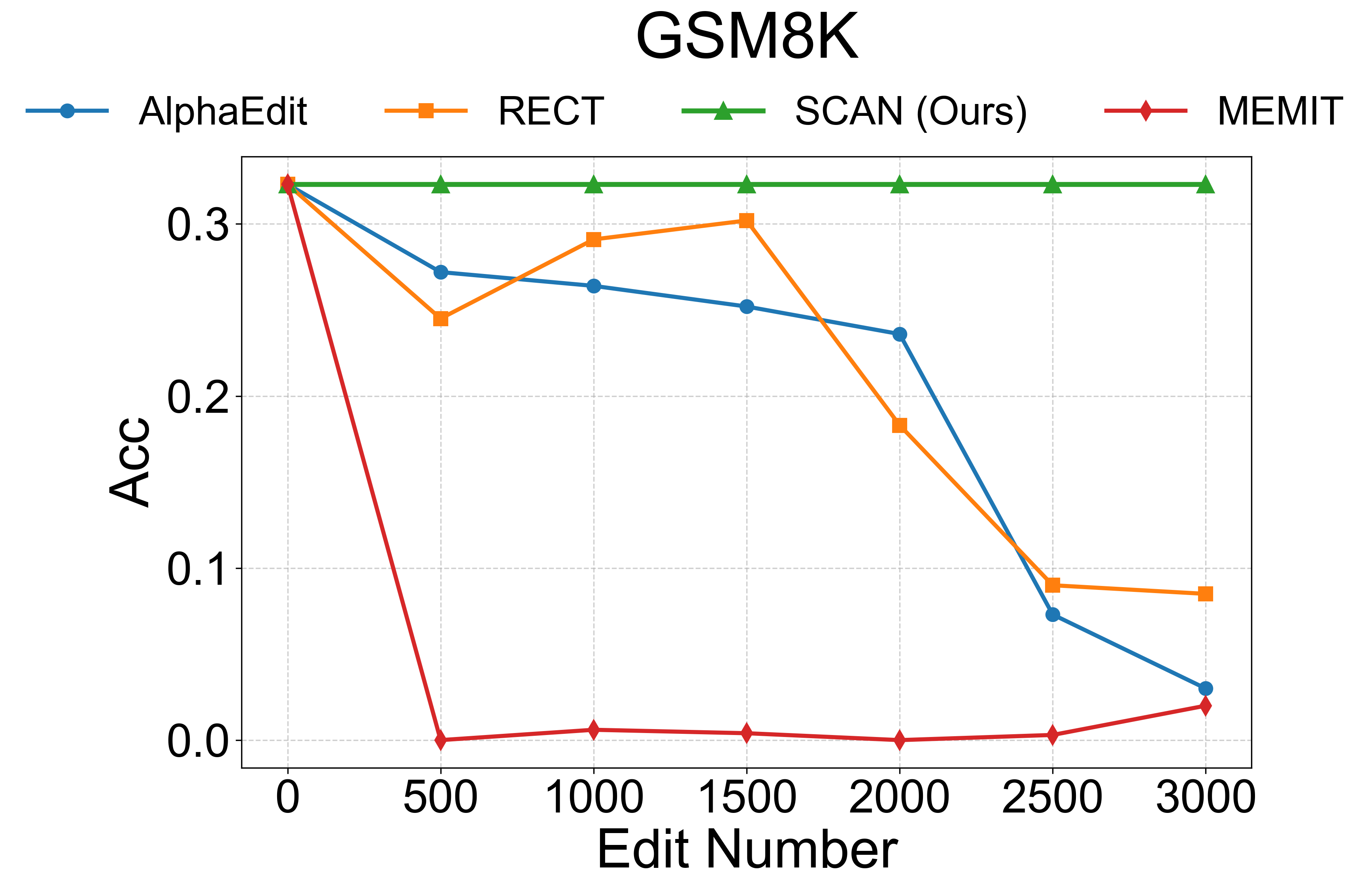}
        \label{fig:gsm}
    }

    \subfigure[]{
        \includegraphics[width=0.29\textwidth,
        trim=15pt 0pt 15pt 0pt,
        clip]{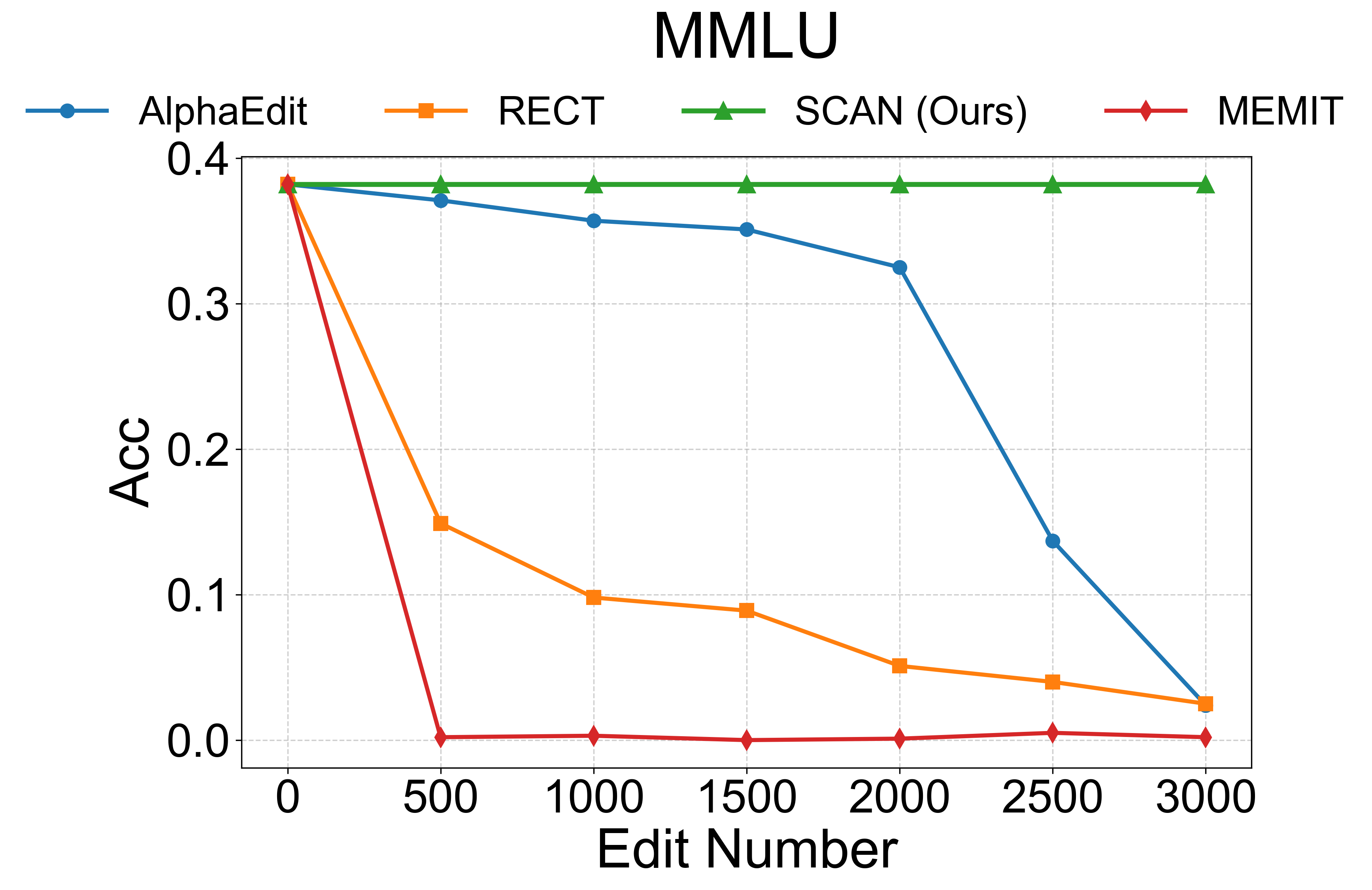}
        \label{fig:mmlu}
    }
\hspace{0.5em}
    \subfigure[]{
        \includegraphics[width=0.29\textwidth,
        trim=15pt 0pt 15pt 0pt,
        clip]{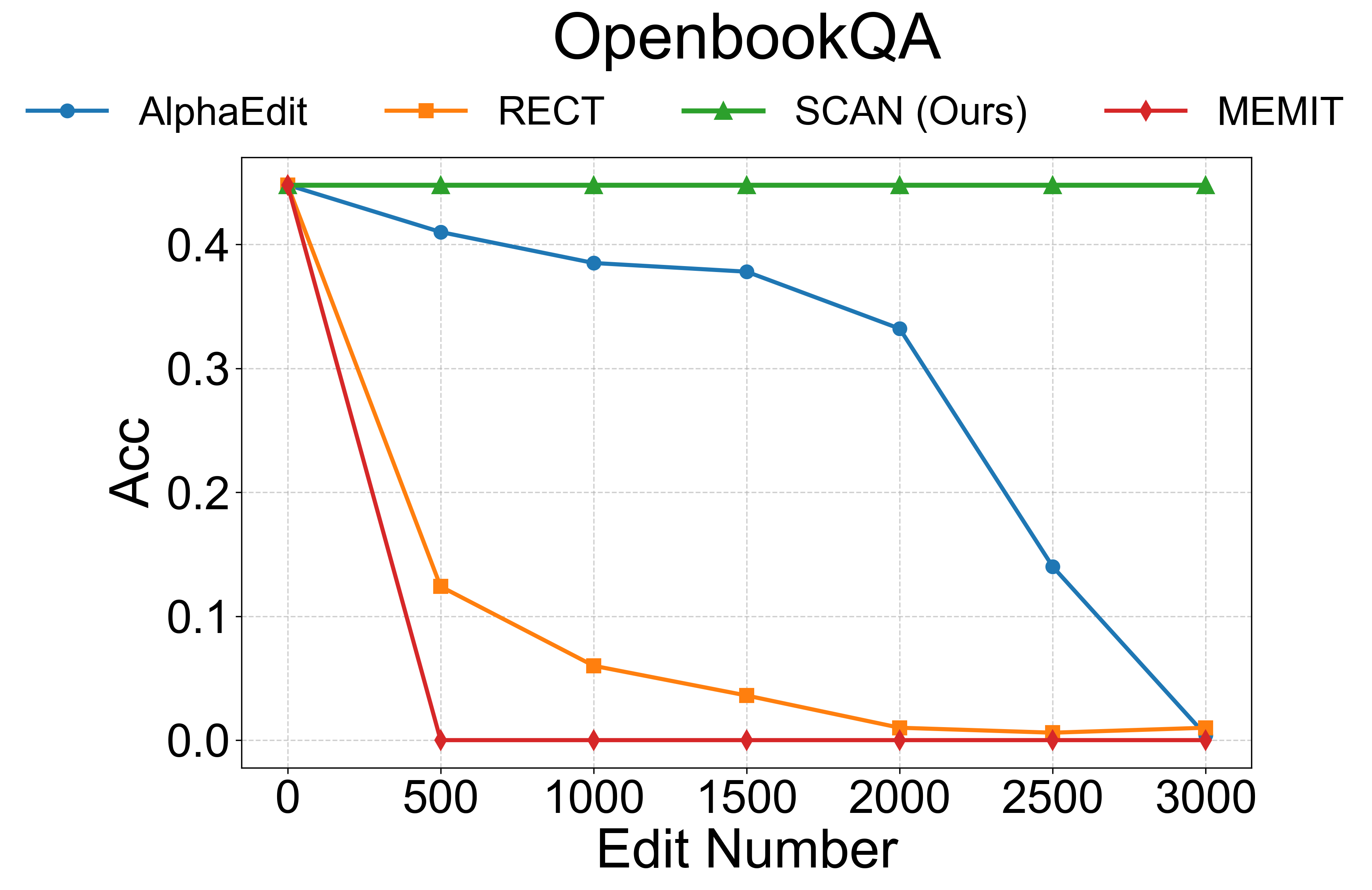}
        \label{fig:obqa}
    }
\hspace{0.5em}
    \subfigure[]{
        \includegraphics[width=0.29\textwidth,
        trim=15pt 0pt 15pt 0pt,
        clip]{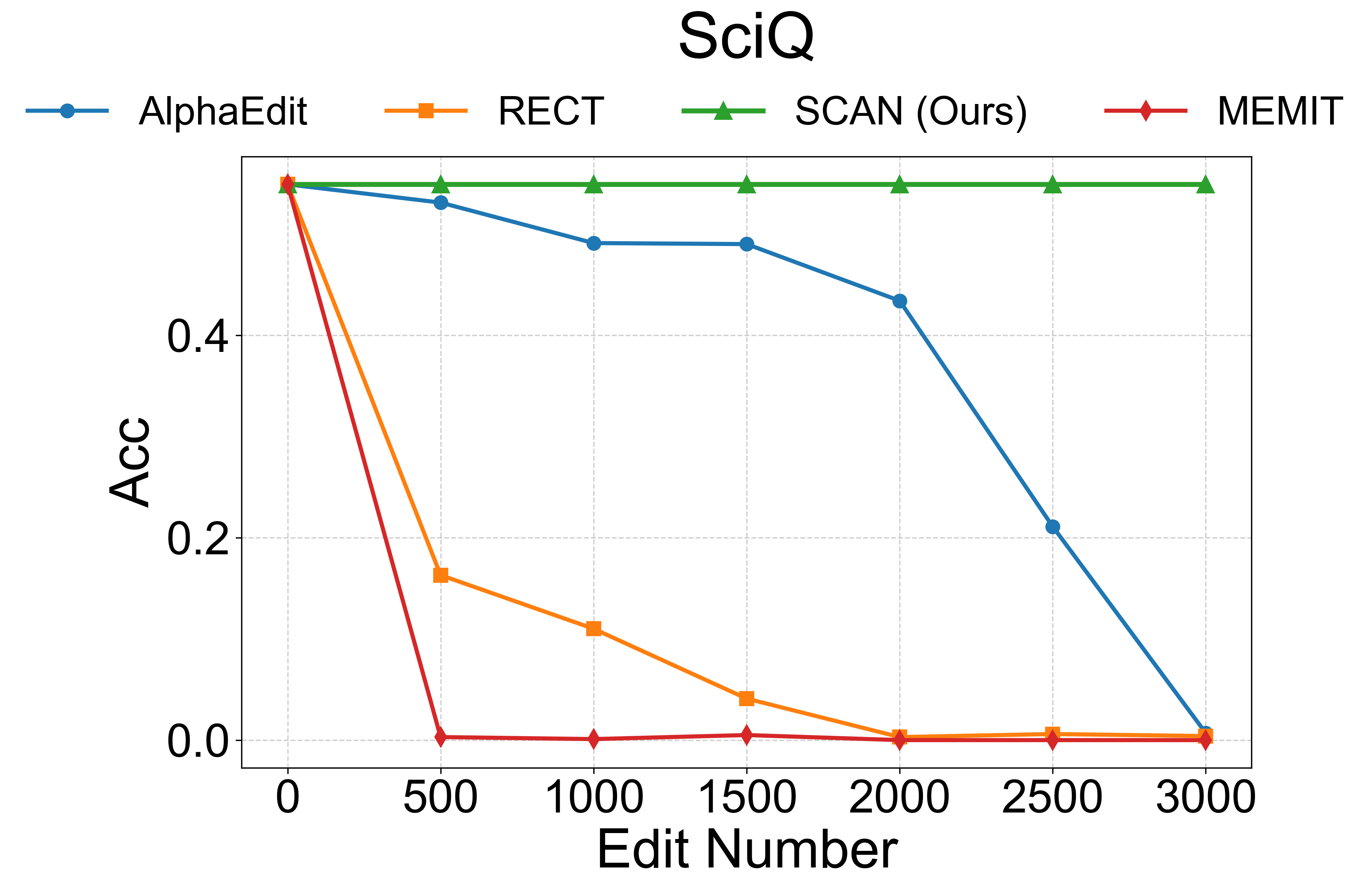}
        \label{fig:sciq}
    }

    \caption{Comparison of SCAN against other methods across six benchmarks. The results show the ability to preserve the model’s general performance as the number of edit cases grows.}
    \label{fig:downstream_benchmarks}
\end{figure*}

\section{Related Work}

\textbf{Model Editing for Knowledge Update.} Current model editing techniques are broadly categorized into parameter-modifying and parameter-preserving methods. Parameter Modifying Methods directly update model weights to encode new facts. Meta-learning approaches like MEND \cite{mitchell2022fast} use hypernetworks to transform gradients into specific parameter updates, while locate-then-edit methods such as ROME \cite{meng2022locating} and MEMIT \cite{meng2022mass} update parameters by solving constrained optimization problems, aiming to balance successful editing with the preservation of unrelated knowledge. AlphaEdit \cite{ICLR2025_29c8c615} employs null-space projection to ensure updates satisfy constraints without interfering with unrelated knowledge. Notably, RECT \cite{gu-etal-2024-model} introduces sparse edit but lacks a deep attribution-based analysis. In contrast, parameter-preserving Methods update model without altering the original weights. GRACE \cite{hartvigsen2023aging} achieves this by caching discrete codebook values for specific hidden states, while MELO \cite{yu2024melo} stores the lora weight. They both use Euclidean Distance to trigger the desired edit. WISE \cite{wang2024wise} utilizes pretrained router to compute activation as a trigger and a ``side memory'' MLP to generate output as perturbation. React \cite{zhong-etal-2025-react} use two pretrained MLPs to calculate semantic similarity and edit by computing belief shift based on positive and negative representations.

\textbf{Mechanistic Interpretability.} Our work is grounded in Mechanistic Interpretability. Early research focused on localization and logit analysis. Causal Tracing identifies important components by patching neuron to observe their causal effect \cite{meng2022locating}. Logit Lens monitors the model's intermediate computations by projecting hidden states onto the vocabulary via unembedding matrix \cite{geva2022transformer,dar2023analyzing}. These foundational techniques have evolved into more sophisticated Circuit Analysis, which seeks to decompose the model's computation into functional paths \cite{ameisen2025circuit}. To resolve the polysemanticity inherent in dense LLMs, Sparse Autoencoders (SAEs) \cite{cunningham2023sparse} and their variants, such as Sparse Transcoders \cite{dunefsky2024transcoders}, have been developed to project activations onto a sparse, monosemantic feature. Unlike standard SAEs, transcoders mimic the behavior of MLP layers, allowing for a more direct mapping of the knowledge storage mechanism \cite{paulo2025transcoders}.

\section{Conclusion}
In this paper, we introduced SCAN, a framework leveraging Sparse Transcoders and Attribution Graphs for sparse editing in a lifelong setting. By isolating minimal features, SCAN overcomes the granularity limitations, reducing side effects on unrelated knowledge. Extensive experiments demonstrate that SCAN achieves superior performance. Ultimately, SCAN contributes an interpretable perspective to the field of model editing.

\section*{Impact Statement}

This paper presents work whose goal is to advance the field of LLMs by improving the interpretability and controllability of LLMs representations, with a particular focus on lifelong knowledge editing. By enabling more transparent identification of model components and more targeted modifications of model behavior, this work may contribute to safer and more reliable LLMs systems. While techniques for model editing could, in principle, be misused to alter factual information in deployed models, our approach emphasizes interpretability and precise, controlled edits, which may help reduce unintended or harmful modifications. We do not anticipate significant negative societal consequences arising directly from this research.


\bibliography{main}
\bibliographystyle{icml2026}

\newpage
\appendix

\onecolumn
\section{Detailed experiment setup}
\subsection{Datasets}
Here, we provide a detailed introduction to the datasets used in this paper:
\begin{itemize}
    \item \textbf{CounterFact}: CounterFact is a more challenging dataset that contrasts counterfactual with
    factual statements, initially scoring lower for CounterFact. It constructs out-of-scope data by replacing the subject entity with approximate entities sharing the same predicate. The CounterFact dataset
    has similar metrics to ZsRE for evaluating efficacy, generalization, and specificity. Additionally,
    CounterFact includes multiple generation prompts with the same meaning as the original prompt to
    test the quality of generated text, specifically focusing on fluency and consistency.
    \item \textbf{ZsRE}: ZsRE is a question answering (QA) dataset that uses questions generated
    through back-translation as equivalent neighbors. Following previous work, natural questions
    are used as out-of-scope data to evaluate locality. Each sample in ZsRE includes a subject string
    and answers as the editing targets to assess editing success, along with the rephrased question for
    generalization evaluation and the locality question for evaluating specificity.
    \item \textbf{WikiFactDiff}: WikiFactDiff is a dataset focused on the task of factual updates, contrasting new, obsolete, and static facts across two different time points. It constructs out-of-scope data by comparing the state of the Wikidata knowledge base on January 4, 2021, and February 27, 2023. The dataset includes various update scenarios such as fact replacements, archiving, and new entity insertions, with facts represented as subject-relation-object triples. WikiFactDiff includes verbalization templates and cloze tests for evaluating update algorithms, specifically focusing on the quality and consistency of updates.
\end{itemize}
\subsection{Metrics}
\paragraph{Evaluation Formulation.}
All datasets in our experiments follow the same evaluation protocol.
Given an edit $e = (s, r, o, o^*)$, we evaluate the editing performance 
\textbf{Reliability}, \textbf{Generality}, \textbf{Locality}
and \textbf{General ability} for general dataset like MMLU.

\paragraph{Reliability.}
Reliability measures whether the model correctly applies the intended edit.
It evaluates the model's ability to produce the edited target $o^*$ when queried with the original subject--relation pair $(s, r)$:
\[
\mathcal{M}_{\text{rel}}
=
\mathbb{E}_{e \sim \mathcal{D}_{\text{edit}}}
\;
\mathbb{I}
\left[
\arg\max_{o}
P_{f^*}\!\left(o \mid p(s, r)\right)
=
o^*
\right]
\]

\paragraph{Generality.}
Generality evaluates whether the applied edit generalizes to in-scope variants of the edited fact. It measures the model's ability to output the same edited target $o^*$ under semantically equivalent prompts associated with the same edit:
\[
\mathcal{M}_{\text{gen}}
=
\mathbb{E}_{e \sim \mathcal{D}_{\text{edit}},{p^* \sim \mathcal{N}(e)}}
\;
\mathbb{I}
\left[
\arg\max_{o}
P_{f^*}\!\left(o \mid p^*(s, r)\right)
=
o^*
\right]
\]
where $\mathcal{N}(e)$ denotes the set of in-scope prompt variations for edit $e$.

\paragraph{Locality.}
Locality measures whether the edit avoids unintended side effects on unrelated knowledge.
It evaluates whether the model's predictions on out-of-scope inputs remain unchanged after editing:
\[
\mathcal{M}_{\text{loc}}
=
\mathbb{E}_{(x,p) \sim \mathcal{D}_{\text{loc}}}
\;
\mathbb{I}
\left[
\arg\max_{x}
P_{f^*}(x \mid p)
=
\arg\max_{x}
P_{f}(x \mid p)
\right]
\]
where $f$ and $f^*$ denote the original and edited models, respectively.

\paragraph{General Ability.}
General ability evaluates whether the editing process degrades the model's overall reasoning and knowledge capabilities across diverse domains.
\[
\mathcal{M}_{\text{ga}}
=
\mathbb{E}_{(x, y) \sim \mathcal{D}_{\text{ga}}}
\;
\mathbb{I}
\left[
\arg\max_{y}
P_{f^*}(y \mid x)
=
y
\right]
\]
where $f^*$ denotes the edited model.

\subsection{Baselines}
Six popular model editing methods were selected as baselines including:
\begin{itemize}
    \item \textbf{FT}: FT simply performed gradient descent on the edits to update model parameters. It fine-tuned the last layer in the model with a norm constraint on weight changes to prevent overfitting.
    \item \textbf{RECT}: RECT is a method that regularizes weight updates based on their relative changes to control parameter perturbations, thereby achieving knowledge edits while maximizing the preservation of the model's general capabilities.
    \item \textbf{GRACE}: GRACE enables localized corrections of streaming errors in deployed models by writing new mappings into the pretrained model's latent space, creating a discrete, local edit cache, thereby achieving continuous knowledge updates while minimizing the impact on unrelated inputs, all without modifying the model weights.
    \item \textbf{MELO}:MELO is a plug‑in model editing method based on neuron‑indexed dynamic LoRA, which alters the behavior of language models by dynamically activating certain LoRA blocks according to an internal vector database
    \item \textbf{MEMIT}: MEMIT first localizes the key positions storing factual knowledge in the Transformer MLP modules, and then simultaneously updates a large set of facts across multiple MLP layers, achieving order-independent knowledge edits with minimal impact on other knowledge.
    \item \textbf{AlphaEdit}: AlphaEdit builds on MEMIT by introducing null-space projection, constraining knowledge updates in directions that do not disrupt existing knowledge, thereby achieving more robust, order-independent multi-fact edits.
\end{itemize}
\subsection{Implementation details}
We report the hyperparameter configurations used for all editing methods and backbone models. 
For each method, we select the best-performing configuration (recommandation from \cite{wang-etal-2024-easyedit}) and apply it consistently across all experiments.

\centering
\noindent\textbf{AlphaEdit}
\smallskip

\centering
\footnotesize
\setlength{\tabcolsep}{8pt}
\begin{tabular}{l|c|c|c}
\hline
\textbf{Parameter} & \textbf{Qwen3-8B} & \textbf{Llama3.1-8B} & \textbf{Gemma2-2B} \\
\hline
Layers & 4--8 & 4--8 & 4--8 \\
Fact token & subj\_last & subj\_last & subj\_last \\
$w_{\text{decay}}$ & $10^{-3}$ & $10^{-3}$ & $10^{-3}$ \\
Mom. weight & 15000 & 15000 & 15000 \\
Mom. samples & 3000 & 3000 & 3000 \\
Clamp factor & 0.75 & 0.75 & 0.75 \\
Null thresh. & 0.02 & 0.02 & 0.02 \\
$L_2$ & 100 & 10 & 500 \\
Steps & 25 & 25 & 25 \\
LR & $10^{-3}$ & $10^{-3}$ & $10^{-3}$ \\
\hline
\end{tabular}

\bigskip
\noindent\textbf{Fine-Tuning (FT)}
\smallskip

\begin{tabular}{l|c|c|c}
\hline
\textbf{Parameter} & \textbf{Qwen3-8B} & \textbf{Llama3.1-8B} & \textbf{Gemma2-2B} \\
\hline
Layers & 35 & 31 & 25 \\
Steps & 25 & 25 & 25 \\
LR & $5\times10^{-4}$ & $5\times10^{-4}$ & $5\times10^{-4}$ \\
\hline
\end{tabular}

\bigskip
\noindent\textbf{MEMIT}
\smallskip

\begin{tabular}{l|c|c|c}
\hline
\textbf{Parameter} & \textbf{Qwen3-8B} & \textbf{Llama3.1-8B} & \textbf{Gemma2-2B} \\
\hline
Layers & 4--8 & 4--8 & 4--8 \\
Fact token & subj\_last & subj\_last & subj\_last \\
$w_{\text{decay}}$ & $10^{-3}$ & $10^{-3}$ & $10^{-3}$ \\
Mom. weight & 15000 & 15000 & 15000 \\
Mom. samples & 3000 & 3000 & 3000 \\
Clamp factor & 4 & 4 & 4 \\
Steps & 25 & 25 & 25 \\
LR & 0.5 & 0.5 & 0.5 \\
\hline
\end{tabular}

\bigskip
\noindent\textbf{RECT}
\smallskip

\begin{tabular}{l|c|c|c}
\hline
\textbf{Parameter} & \textbf{Qwen3-8B} & \textbf{Llama3.1-8B} & \textbf{Gemma2-2B} \\
\hline
Layers & 4 & 4 & 4 \\
Sparse rate & 0.2 & 0.2 & 0.2 \\
Reg. & abs & abs & abs \\
Steps & 25 & 25 & 25 \\
LR & 0.5 & 0.5 & 0.5 \\
\hline
\end{tabular}

\bigskip
\noindent\textbf{GRACE}
\smallskip

\begin{tabular}{l|c|c|c}
\hline
\textbf{Parameter} & \textbf{Qwen3-8B} & \textbf{Llama3.1-8B} & \textbf{Gemma2-2B} \\
\hline
Layers & 23 & 27 & 19 \\
$\epsilon$ & 1 & 1 & 1 \\
Steps & 100 & 100 & 100 \\
LR & 1 & 1 & 1 \\
\hline
\end{tabular}

\bigskip
\noindent\textbf{MELO}
\smallskip

\begin{tabular}{l|c|c|c}
\hline
\textbf{Parameter} & \textbf{Qwen3-8B} & \textbf{Llama3.1-8B} & \textbf{Gemma2-2B} \\
\hline
Layers & 35,35 & 30,31 & 24,25 \\
Steps & 50 & 50 & 50 \\
LR & $10^{-4}$ & $10^{-4}$ & $10^{-4}$ \\
LoRA $(r,\alpha)$ & (64,64) & (64,64) & (64,64) \\
Radius & 75 & 75 & 75 \\
\hline
\end{tabular}

\bigskip
\noindent\textbf{SCAN (Ours)}
\smallskip

\begin{tabular}{l|c|c|c}
\hline
\textbf{Parameter} & \textbf{Qwen3-8B} & \textbf{Llama3.1-8B} & \textbf{Gemma2-2B} \\
\hline
Layers & all & all & all \\
Steps & 50 & 50 & 50 \\
LR & 0.005 & 0.005 & 0.005 \\
Thresh. & 0.25 & 0.25 & 0.25 \\
Features Num.& 300 & 450 & 200 \\
Node Thresh. & 0.8 & 0.9 & 0.8 \\
Edge Thresh. & 0.98 & 0.99 & 0.98 \\
Node Num. & 8192 & 8192 & 8192 \\
Transcoder Dim. & 163840 & 131072 & 16384 \\
\hline
\end{tabular}

\section{Detailed Proof}
\subsection{Jacobian as the Optimal Direction-Preserving Linearization}
\begin{lemma}[Stability of normalization]
\label{lem:normalization}
Let $u,v\in\mathbb{R}^n$ be nonzero vectors. 
Then
\[
\left\|
\frac{u}{\|u\|}
-
\frac{v}{\|v\|}
\right\|
\le
\frac{2\|u-v\|}{\|v\|}
\]
\end{lemma}
\begin{proof}
We decompose
\[
\frac{u}{\|u\|}-\frac{v}{\|v\|}
=
\frac{u-v}{\|u\|}
+
v\!\left(\frac{1}{\|u\|}-\frac{1}{\|v\|}\right)
\]
Taking norms and applying the triangle inequality yields
\[
\left\|
\frac{u}{\|u\|}-\frac{v}{\|v\|}
\right\|
\le
\frac{\|u-v\|}{\|u\|}
+
\|v\|\left|\frac{1}{\|u\|}-\frac{1}{\|v\|}\right|
\]

Using the reverse triangle inequality,
\(
|\|u\|-\|v\||\le \|u-v\|
\)
we obtain
\[
\|v\|\left|\frac{1}{\|u\|}-\frac{1}{\|v\|}\right|
=
\frac{|\|v\|-\|u\||}{\|u\|}
\le
\frac{\|u-v\|}{\|u\|}
\]
Hence,
\[
\left\|
\frac{u}{\|u\|}-\frac{v}{\|v\|}
\right\|
\le
\frac{2\|u-v\|}{\|u\|}
\]

\end{proof}
\begin{proof}
Fix $x_0 \in X$ and assume that $f$ is differentiable at $x_0$ with
Jacobian matrix $J_f(x_0)$. By differentiability at $x_0$, we have
\[
f(x_0+h)
=
f(x_0) + J_f(x_0)h + r_{x_0}(h),
\qquad
\frac{\|r_{x_0}(h)\|}{\|h\|}\xrightarrow[h\to 0]{}0
\]

Setting $h=-x_0$ yields
\[
f(0)
=
f(x_0) - J_f(x_0)x_0 + r_{x_0}(-x_0)
\]
Since $f(0)=0$, this can be rewritten as
\[
f(x_0)
=
J_f(x_0)x_0 - r_{x_0}(-x_0)
\]

Define $\tilde r(x_0):=-r_{x_0}(-x_0)$. Then we have
\[
f(x_0) = J_f(x_0)x_0 + \tilde r(x_0)
\qquad \text{and} \qquad
\frac{\|\tilde r(x_0)\|}{\|x_0\|}
=
\frac{\|r_{x_0}(-x_0)\|}{\|x_0\|}
\xrightarrow[x_0\to 0]{}0
\]

Then we consider the difference in normalized directions:
\[
\left\|
\frac{f(x_0)}{\|f(x_0)\|}
-
\frac{J_f(x_0)x_0}{\|J_f(x_0)x_0\|}
\right\|
=
\left\|
\frac{J_f(x_0)x_0+\tilde r(x_0)}{\|J_f(x_0)x_0+\tilde r(x_0)\|}
-
\frac{J_f(x_0)x_0}{\|J_f(x_0)x_0\|}
\right\|
\]

By the lemma ~\ref{lem:normalization}, we have
\[
\left\|
\frac{f(x_0)}{\|f(x_0)\|}
-
\frac{J_f(x_0)x_0}{\|J_f(x_0)x_0\|}
\right\|
\le
\frac{2\|\tilde r(x_0)\|}{\|J_f(x_0)x_0\|}
\]

Since $J_f(0)$ is non-singular, there exists a constant $M > 0$ such that $\|J_f(0)x_0\| \ge M\|x_0\|$. By the continuity of the Jacobian at the origin, we have $J_f(x_0) \to J_f(0)$ as $x_0 \to 0$. For sufficiently small $x_0$, the reverse triangle inequality implies:
\[
\|J_f(x_0)x_0\| \ge \|J_f(0)x_0\| - \|(J_f(x_0) - J_f(0))x_0\| \ge \frac{M}{2}\|x_0\|
\]
Setting $m = M/2$, it follows that $\|J_f(x_0)x_0\| \ge m\|x_0\|$ for $x_0$ near the origin. This implies that the ratio satisfies
\[
\frac{2\|\tilde r(x_0)\|}{\|J_f(x_0)x_0\|}
\le
\frac{2\|\tilde r(x_0)\|}{m\|x_0\|}
=
\frac{2}{m} \cdot \frac{\|\tilde r(x_0)\|}{\|x_0\|}
\xrightarrow[x_0\to 0]{}0
\]

Hence, we obtain
\[
\left\|
\frac{f(x_0)}{\|f(x_0)\|}
-
\frac{J_f(x_0)x_0}{\|J_f(x_0)x_0\|}
\right\|
\xrightarrow[x_0\to 0]{}0
\]
which proves that the Jacobian matrix $J_f(x_0)$ satisfies the desired
directional alignment property.
\end{proof}

\subsection{Closed-form Total Attribution Matrix}
\begin{lemma}[Convergence of Powers of $A$]
Let $A \in \mathbb{R}^{n \times n}$ satisfy $\|A\| < 1$. Then
\[
A^k \xrightarrow[k \to \infty]{} 0
\]
where $0$ is the $n \times n$ zero matrix.
\end{lemma}

\begin{proof}
For any positive integer $k$, we have
\[
\|A^k\| = \|A \cdot A^{k-1}\| \le \|A\| \cdot \|A^{k-1}\| \le \cdots \le \|A\|^k
\]
Since $\|A\| < 1$. Hence,
\[
\|A^k\| \xrightarrow[k \to \infty]{} 0
\]
which implies
\[
A^k \xrightarrow[k \to \infty]{} 0
\]
\end{proof}

\begin{proof}[Proof of Invertibility]
To prove that $I-A$ is invertible, it suffices to show that $0$ is not an eigenvalue of $I-A$.
Equivalently, we only need to show that $1$ is not an eigenvalue of $A$.

We argue by contradiction. 
Suppose that $1$ is an eigenvalue of $A$.
Then there exists a nonzero vector $x \in \mathbb{R}^n$ such that
\[
Ax = x
\]
Without loss of generality, we may assume $\|x\| = 1$.

By iteration, for any positive integer $k$ we have
\[
A^k x = x
\]
Taking norms yields
\[
1 = \|x\| = \|A^k x\| \le \|A^k\|\,\|x\| = \|A^k\|
\]

On the other hand, if $\|A\| < 1$, then
\[
\|A^k\| \le \|A\|^k \xrightarrow[k \to \infty]{} 0
\]
This contradicts the inequality $1 \le \|A^k\|$.

Therefore, $1$ cannot be an eigenvalue of $A$, and hence $0$ is not an eigenvalue of $I-A$.
It follows that $I-A$ is invertible.
\end{proof}

\begin{proof}

For any positive integer $k$, a direct computation shows that
\[
(I + A + A^2 + \cdots + A^k)(I - A) = I - A^{k+1}
\]
Consequently,
\[
I + A + A^2 + \cdots + A^k
= (I - A)^{-1} - A^{k+1}(I - A)^{-1}
\]

Since
\[
A^{k+1} \xrightarrow[k \to \infty]{} 0
\]
it follows that
\[
A^{k+1}(I - A)^{-1} \xrightarrow[k \to \infty]{} 0
\]
Letting $k \to \infty$, we conclude that the matrix series
\[
\sum_{k=0}^{\infty} A^k
\]
converges and satisfies
\[
\sum_{k=0}^{\infty} A^k = (I - A)^{-1}
\]
\end{proof}

\section{Illustration of attribution graph}
\subsection{Illustrate algorithm of SCAN}
\raggedright
The complete procedure for constructing Attribution Graph is summarized in Algorithm \ref{alg:graph_pruning}. Owing to the nilpotent structure of the adjacency matrix A, which arises when the nodes are arranged in order of increasing layer, (i.e., from lower to higher levels, such that $a_{i,j}=0$ for all $i\geq j$) and the fact that the maximum path length in multi-step attributions is finite, we adopt an iterative algorithm to compute total attribution matrix $B$. Moreover the sparsity of $A$ ensures computational efficiency.

\begin{algorithm}[H]
\caption{Knowledge-Specific Circuit Construction and Pruning}
\label{alg:graph_pruning}
\begin{algorithmic}

\STATE {\bfseries Input:} editing instance $e=(s,r,o \rightarrow o^*)$, model $M$, prompt $p$, threshold $\tau$, propagation steps $N$

\STATE {\bfseries Forward Pass and Node Collection}
\STATE Run a forward pass of $M$ with prompt $p$
\STATE Record embeddings, Sparse Transcoder activations, MLP reconstruction errors, and output logits
\STATE Construct node set
\[
V = V_{\text{embed}} \cup V_{\text{feature}} \cup V_{\text{error}} \cup V_{\text{logit}}
\]
\STATE Initialize complete Attribution Araph $G = K(V)$

\STATE {\bfseries Direct Attribution Matrix Construction}
\FORALL{node pairs $(u,v)$ with $u$ preceding $v$}
    \STATE Obtain activation $z_u$
    \STATE Backpropagate gradient from node $v$ and compute $\frac{\partial{M_v}}{\partial{z_u}} $ 
    \STATE Set $A_{v,u} \leftarrow \frac{\partial{M_v}}{\partial{z_u}} \cdot z_u$
\ENDFOR

\STATE {\bfseries Iterative Indirect Attribution Accumulation}
\STATE Initialize $A_1 \leftarrow A$
\FOR{$k = 2$ {\bfseries to} $N$}
    \STATE $A_k \leftarrow A \cdot A_{k-1} + A$
\ENDFOR
\STATE Set $B \leftarrow A_N$

\STATE {\bfseries Recursive Edge Pruning}
\FORALL{target nodes $v \in V$}
    \STATE Collect incoming edges $\mathcal{E}_v = \{u \to v\}$
    \STATE Sort $\mathcal{E}_v$ in descending order of $B_{v,u}$
    \STATE Normalize scores $\widetilde{B}_{uv}$
    \STATE Initialize cumulative sum $c \leftarrow 0$
    \FORALL{edges $u \to v$ in sorted order}
        \STATE $c \leftarrow c + \widetilde{B}_{v,u}$
        \IF{$c \ge \tau$}
            \STATE Prune this edge and all remaining edges in $\mathcal{E}_v$
            \STATE {\bfseries break}
        \ENDIF
    \ENDFOR
\ENDFOR

\STATE {\bfseries Output}

\STATE Return the remaining graph $G'$

\end{algorithmic}
\end{algorithm}

\subsection{Example of Attribution Graph}

To further elucidate the model's decision-making process, we provide a high-fidelity Attribution Graph as an example in Figure \ref{fig:feature_flow}.

\begin{figure}[htbp]
    \centering
    \includegraphics[width=1.0\linewidth, 
                     height=0.4\textheight, 
                     keepaspectratio, 
                     trim=1pt 1pt 1pt 1pt, 
                     clip]               
                     {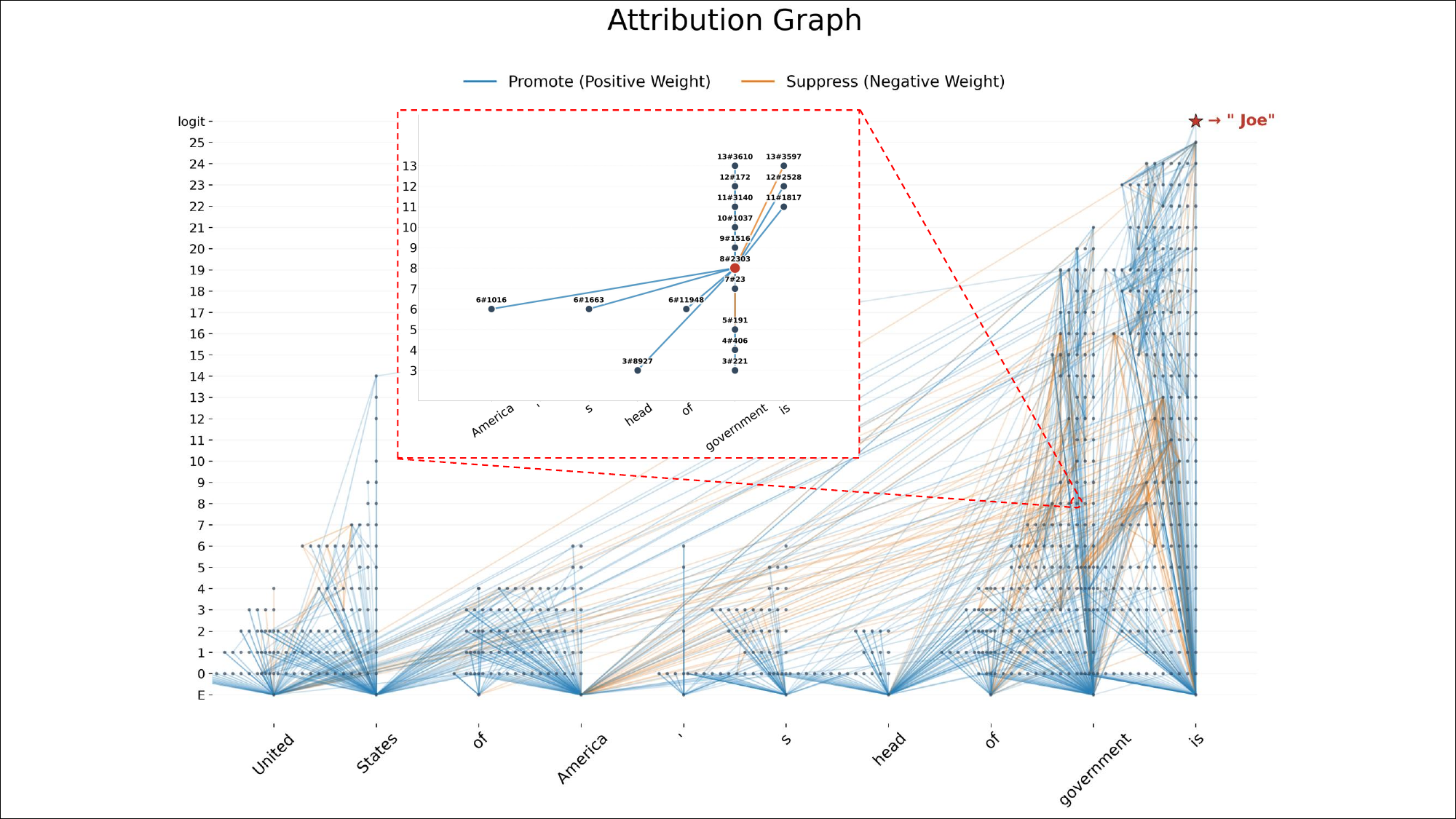}
    \caption{Attribution Graph for the target token ``Joe'' on Gemma2-2B. Nodes represent features across layers. Blue edges denote promoting effects (positive weights), while orange edges denote suppressing effects (negative weights). The color opacity represents the relative magnitude of the attribution weight, with uniform line thickness for visual clarity.}
    \label{fig:feature_flow}
\end{figure}

\section{Other experiment} 
\subsection{Result on Llama3.2-1B}
\raggedright
To evaluate the robustness of editing methods on lightweight architectures, we further report results on the small-scale model Llama3.2-1B as shown in Table \ref{tab:llama3.2-1b}.
\begin{table}[H]
\centering
\caption{Sequential editing task performance comparison of our method and other methods after 1000 edits. \textbf{Bold} and \underline{underline} denote the best and second-best results per column, respectively.}
\label{tab:llama3.2-1b}

\renewcommand{\arraystretch}{1.0}
\setlength{\tabcolsep}{10pt}
\newcommand{\lighttext}[1]{{\color{black!70}#1}}

\resizebox{\textwidth}{!}{%
\Large 
\begin{tabular}{cc cccc | cccc | cccc} 

\toprule[1.5pt]
\multirow{2}{*}{\textbf{Method}} & \multirow{2}{*}{\textbf{Model}} &
\multicolumn{4}{c}{\textbf{CounterFact}} &
\multicolumn{4}{c}{\textbf{ZsRE}} &
\multicolumn{4}{c}{\textbf{WikiFactDiff}} \\
\cmidrule(lr){3-6} \cmidrule(lr){7-10} \cmidrule(lr){11-14}
& & \textbf{Rel} $\uparrow$ & \textbf{Gen} $\uparrow$ & \textbf{Loc} $\uparrow$ & \textbf{Avg} $\uparrow$
& \textbf{Rel} $\uparrow$ & \textbf{Gen} $\uparrow$ & \textbf{Loc} $\uparrow$ & \textbf{Avg} $\uparrow$
& \textbf{Rel} $\uparrow$ & \textbf{Gen} $\uparrow$ & \textbf{Loc} $\uparrow$ & \textbf{Avg} $\uparrow$ \\
\midrule[1pt]

FT & \multirow{7}{*}{\rotatebox{90}{Llama3.2-1B}} & \lighttext{32.25} & \lighttext{11.05} & \lighttext{1.90} & \lighttext{5.23}
& \lighttext{55.97} & \lighttext{48.16} & \lighttext{6.01} & \lighttext{14.86}
& \lighttext{68.17} & \lighttext{64.81} & \lighttext{16.95} & \lighttext{30.98} \\
RECT & & \lighttext{2.30} & \lighttext{4.10} & \lighttext{0} & \lighttext{0}
& \lighttext{1.50} & \lighttext{1.33} & \lighttext{1.56} & \lighttext{1.46}
& \lighttext{0.44} & \lighttext{0.37} & \lighttext{0} & \lighttext{0} \\
AlphaEdit & & \lighttext{86.30} & \lighttext{42.95} & \lighttext{23.70} & \lighttext{39.61}
& \lighttext{80.66} & \lighttext{66.36} & \lighttext{45.70} & \lighttext{61.21}
& \lighttext{48.06} & \lighttext{41.45} & \lighttext{24.60} & \lighttext{35.15} \\
MEMIT & & \lighttext{0} & \lighttext{0} & \lighttext{0.30} & \lighttext{0}
& \lighttext{0.23} & \lighttext{0.20} & \lighttext{4.50} & \lighttext{0.55}
& \lighttext{0.15} & \lighttext{0.15} & \lighttext{0.76} & \lighttext{0.24} \\
GRACE & & \textbf{100} & \lighttext{0.80} & \textbf{99.80} & \lighttext{2.37}
& \lighttext{\underline{99.40}} & \lighttext{24.20} & \textbf{100} & \lighttext{42.02}
& \lighttext{\underline{99.82}} & \lighttext{48.30} & \lighttext{\underline{98.84}} & \lighttext{70.38} \\
MELO & & \lighttext{91.60} & \lighttext{\underline{62.15}} & \lighttext{50.90} & \lighttext{\underline{66.01}}
& \lighttext{94.59} & \lighttext{\underline{69.80}} & \lighttext{95.02} & \lighttext{\underline{85.07}}
& \lighttext{86.54} & \lighttext{\underline{70.31}} & \lighttext{94.60} & \lighttext{\underline{82.35}} \\
\rowcolor{red!10}
\textbf{SCAN (Ours)} & & \lighttext{\underline{99.10}} & \textbf{82.60} & \lighttext{\underline{92.40}} & \textbf{90.96}
& \textbf{99.80} & \textbf{79.60} & \lighttext{\underline{99.80}} & \textbf{90.32}
& \textbf{100} & \textbf{81.33} & \textbf{92.01} & \textbf{90.62} \\
\bottomrule[1.5pt]
\end{tabular}
}
\end{table}

\section{Case study}

\subsection{More Interpretable Features}

We select and visualize the activation patterns of several additional interpretable features. These cases encompass both \textit{edit-specific} features tied to particular entities and \textit{general-purpose} features capturing broader semantic abstractions.
\begin{minipage}{\textwidth}
    \noindent\rule{\textwidth}{\lightrulewidth}
    \vspace{1mm}
    \begin{minipage}[c]{0.35\textwidth}
        \raggedright
        \textbf{ID:} 20\#13390 \\
        \textbf{Explanation:} \\
        {\footnotesize References to the company Amazon and/or Amazon branded products and services.}
    \end{minipage}
    \hfill
    \begin{minipage}[c]{0.62\textwidth}
        \centering
        \vspace{3.2mm}
        {\tiny \textsf{Reliability}} \\
        \includegraphics[width=\textwidth, height=1.15cm, keepaspectratio, trim=0 0.115cm 0 0.115cm, clip]{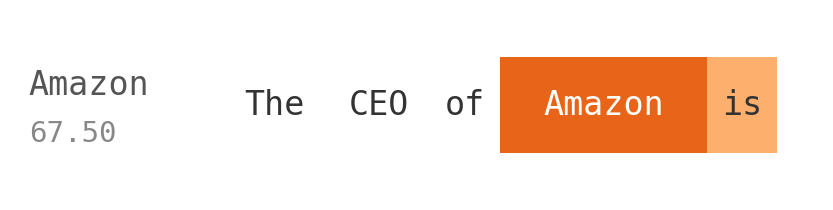} \\
        {\tiny \textsf{Generality}} \\
        \includegraphics[width=\textwidth, height=1.15cm, keepaspectratio, trim=0 0.115cm 0 0.115cm, clip]{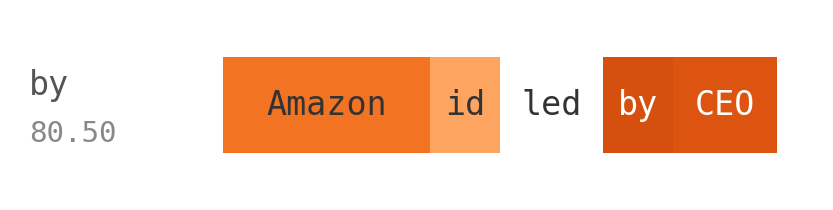} \\
        {\tiny \textsf{Locality}} \\
        \includegraphics[width=\textwidth, height=1.15cm, keepaspectratio, trim=0 0.115cm 0 0.115cm, clip]{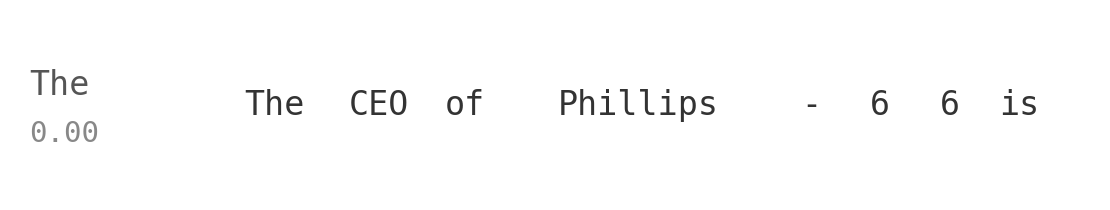}
    \end{minipage}
    \noindent\rule{\textwidth}{\lightrulewidth}
\end{minipage}

\noindent
\begin{minipage}{\textwidth}
    \noindent\rule{\textwidth}{\lightrulewidth}
    \vspace{1mm}
    \begin{minipage}[c]{0.35\textwidth}
        \raggedright
        \textbf{ID:} 20\#10092 \\
        \textbf{Explanation:} \\
        {\footnotesize Mentions of the president of the United States, particularly Obama and Trump, and political terms.}
    \end{minipage}
    \hfill
    \begin{minipage}[c]{0.62\textwidth}
        \centering
        \vspace{3.2mm}
        {\tiny \textsf{Reliability}} \\
        \includegraphics[width=\textwidth, height=1.15cm, keepaspectratio, trim=0 0.115cm 0 0.115cm, clip]{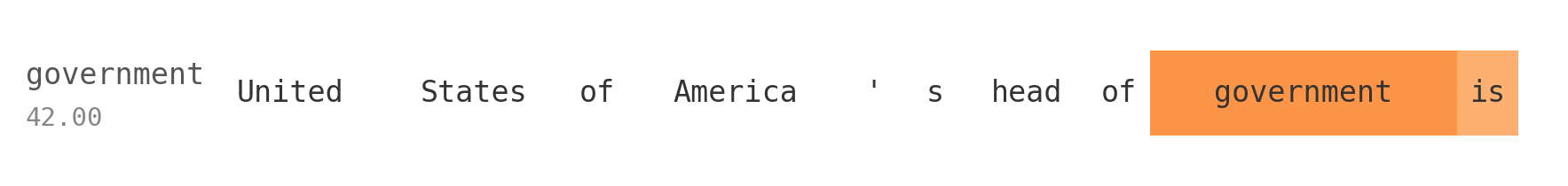} \\
        {\tiny \textsf{Generality}} \\
        \includegraphics[width=\textwidth, height=1.15cm, keepaspectratio, trim=0 0.115cm 0 0.115cm, clip]{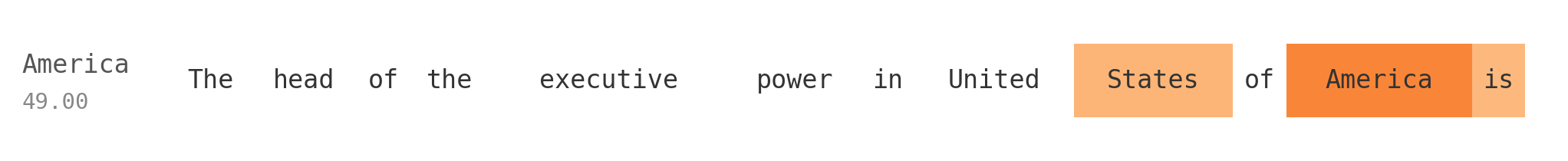} \\
        {\tiny \textsf{Locality}} \\
        \includegraphics[width=\textwidth, height=1.15cm, keepaspectratio, trim=0 0.115cm 0 0.115cm, clip]{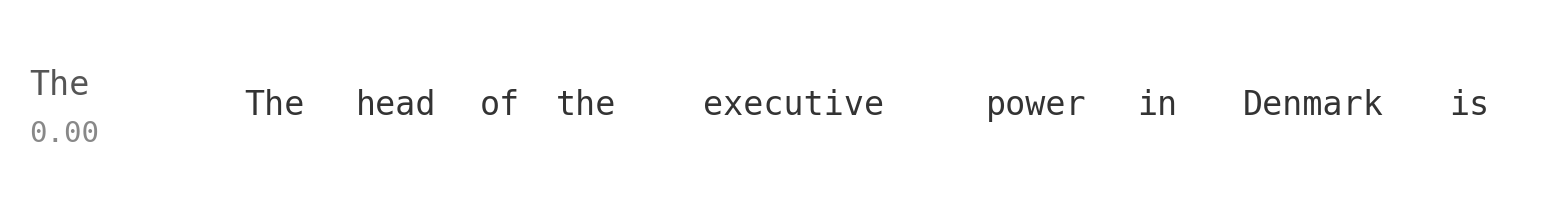}
    \end{minipage}
    \noindent\rule{\textwidth}{\lightrulewidth}
\end{minipage}

\noindent
\begin{minipage}{\textwidth}
    \noindent\rule{\textwidth}{\lightrulewidth}
    \vspace{1mm}
    \begin{minipage}[c]{0.35\textwidth}
        \raggedright
        \textbf{ID:} 21\#3435 \\
        \textbf{Explanation:} \\
        {\footnotesize Capital-letter abbreviations for official-sounding organizations or locations and titles for people.}
    \end{minipage}
    \hfill
    \begin{minipage}[c]{0.62\textwidth}
        \centering
        \vspace{3.2mm}
        {\tiny \textsf{Reliability}} \\
        \includegraphics[width=\textwidth, height=1.15cm, keepaspectratio, trim=0 0.115cm 0 0.115cm, clip]{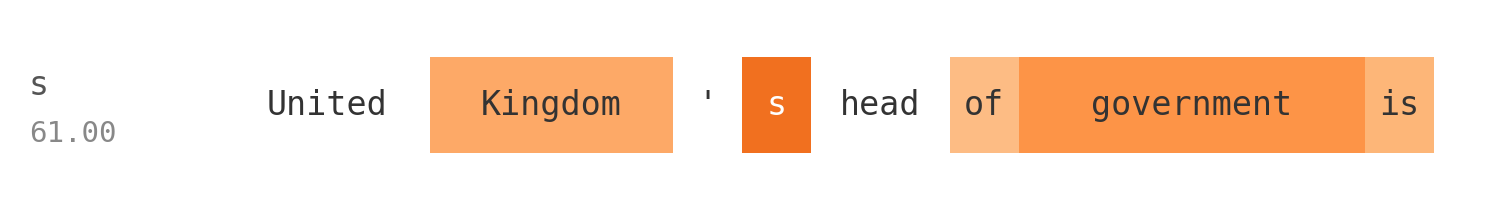} \\
        {\tiny \textsf{Generality}} \\
        \includegraphics[width=\textwidth, height=1.15cm, keepaspectratio, trim=0 0.115cm 0 0.115cm, clip]{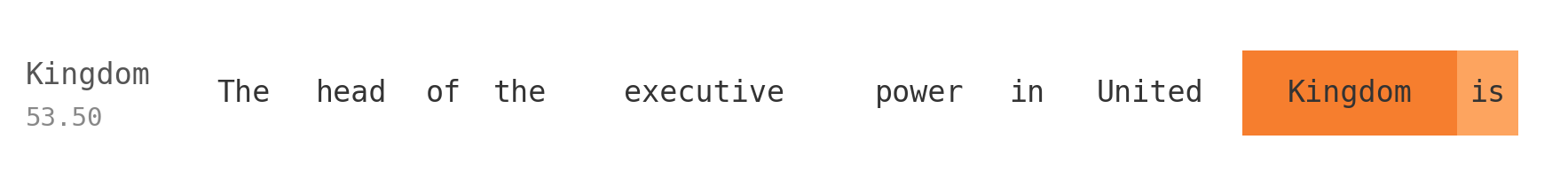} \\
        {\tiny \textsf{Locality}} \\
        \includegraphics[width=\textwidth, height=1.15cm, keepaspectratio, trim=0 0.115cm 0 0.115cm, clip]{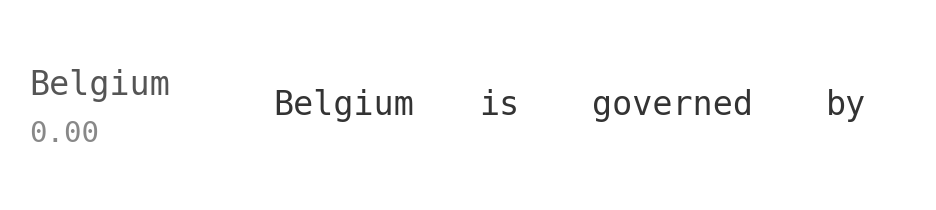}
    \end{minipage}
    \noindent\rule{\textwidth}{\lightrulewidth}
\end{minipage}

\noindent
\begin{minipage}{\textwidth}
    \noindent\rule{\textwidth}{\lightrulewidth}
    \vspace{1mm}
    \begin{minipage}[c]{0.35\textwidth}
        \raggedright
        \textbf{ID:} 19\#15383 \\
        \textbf{Explanation:} \\
        {\footnotesize References to athletes, specifically professional chess players, and terms related to sports participation.}
    \end{minipage}
    \hfill
    \begin{minipage}[c]{0.62\textwidth}
        \centering
        \vspace{3.2mm}
        {\tiny \textsf{Reliability}} \\
        \includegraphics[width=\textwidth, height=1.15cm, keepaspectratio, trim=0 0.115cm 0 0.115cm, clip]{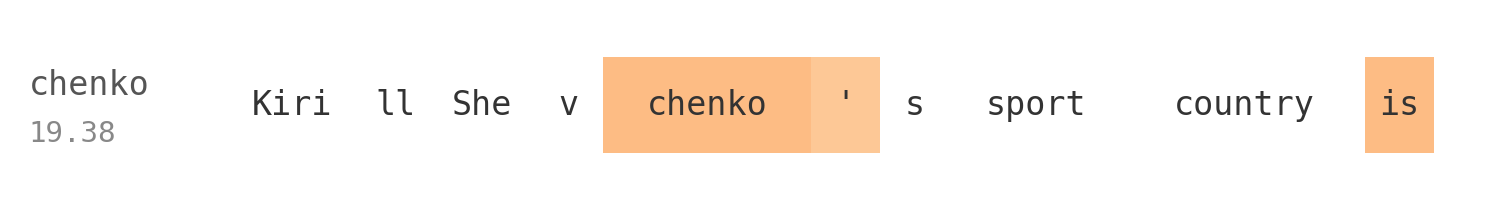} \\
        {\tiny \textsf{Generality}} \\
        \includegraphics[width=\textwidth, height=1.15cm, keepaspectratio, trim=0 0.115cm 0 0.115cm, clip]{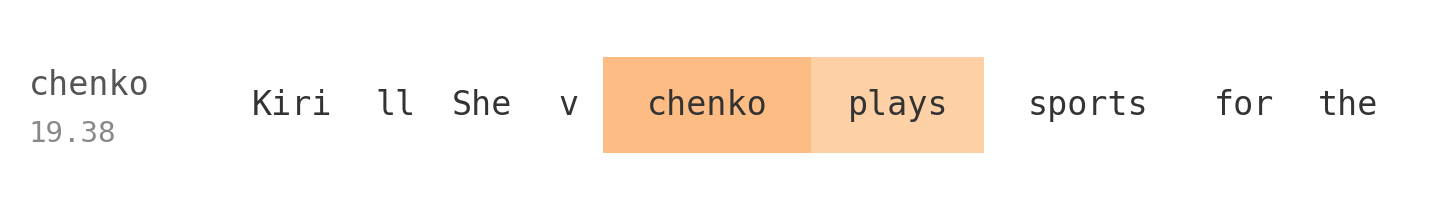} \\
        {\tiny \textsf{Locality}} \\
        \includegraphics[width=\textwidth, height=1.15cm, keepaspectratio, trim=0 0.115cm 0 0.115cm, clip]{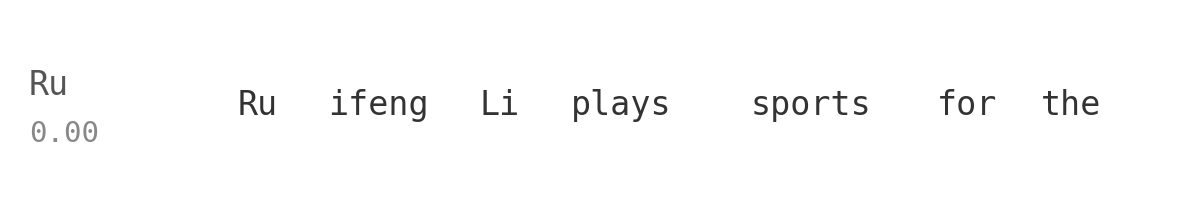}
    \end{minipage}
    \noindent\rule{\textwidth}{\lightrulewidth}
\end{minipage}

\noindent
\begin{minipage}{\textwidth}
    \noindent\rule{\textwidth}{\lightrulewidth}
    \vspace{1mm}
    \begin{minipage}[c]{0.35\textwidth}
        \raggedright
        \textbf{ID:} 19\#10892 \\
        \textbf{Explanation:} \\
        {\footnotesize Mentions of romantic relationships, partners, and well-known couples.}
    \end{minipage}
    \hfill
    \begin{minipage}[c]{0.62\textwidth}
        \centering
        \vspace{3.2mm}
        {\tiny \textsf{Reliability}} \\
        \includegraphics[width=\textwidth, height=1.15cm, keepaspectratio, trim=0 0.115cm 0 0.115cm, clip]{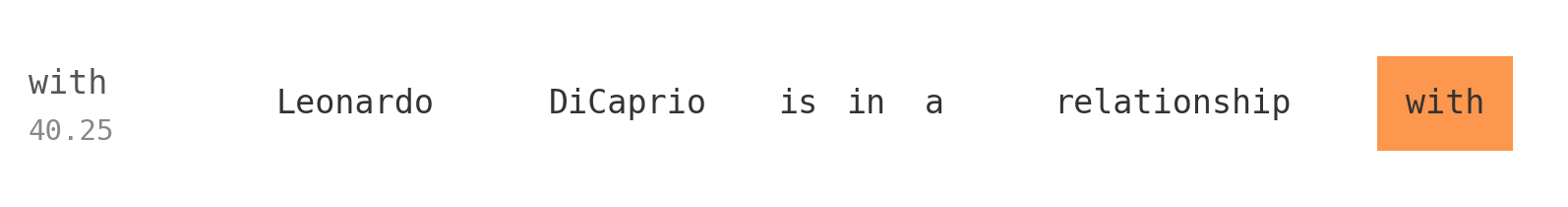} \\
        {\tiny \textsf{Generality}} \\
        \includegraphics[width=\textwidth, height=1.15cm, keepaspectratio, trim=0 0.115cm 0 0.115cm, clip]{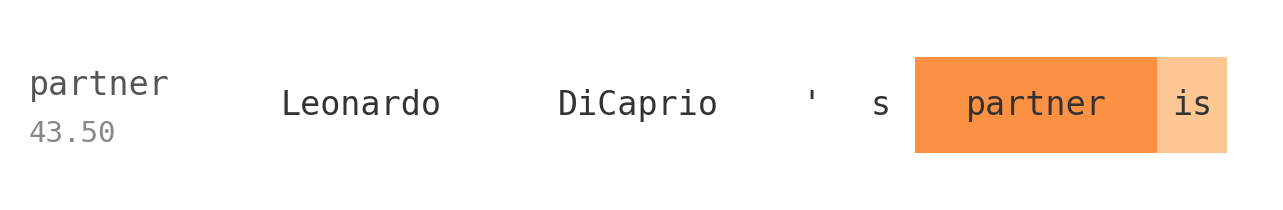} \\
        {\tiny \textsf{Locality}} \\
        \includegraphics[width=\textwidth, height=1.15cm, keepaspectratio, trim=0 0.115cm 0 0.115cm, clip]{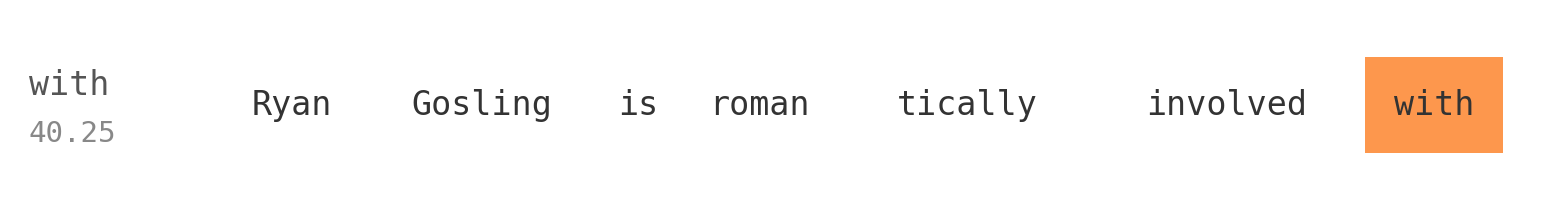}
    \end{minipage}
    \noindent\rule{\textwidth}{\lightrulewidth}
\end{minipage}

\noindent
\begin{minipage}{\textwidth}
    \noindent\rule{\textwidth}{\lightrulewidth}
    \vspace{1mm}
    \begin{minipage}[c]{0.35\textwidth}
        \raggedright
        \textbf{ID:} 22\#1328 \\
        \textbf{Explanation:} \\
        {\footnotesize References to manufacturing, corporate production, and specific automotive brands like Nissan.}
    \end{minipage}
    \hfill
    \begin{minipage}[c]{0.62\textwidth}
        \centering
        \vspace{3.2mm}
        {\tiny \textsf{Reliability}} \\
        \includegraphics[width=\textwidth, height=1.15cm, keepaspectratio, trim=0 0.115cm 0 0.115cm, clip]{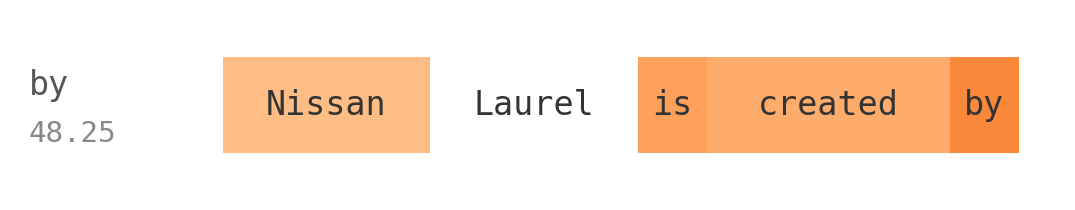} \\
        {\tiny \textsf{Generality}} \\
        \includegraphics[width=\textwidth, height=1.15cm, keepaspectratio, trim=0 0.115cm 0 0.115cm, clip]{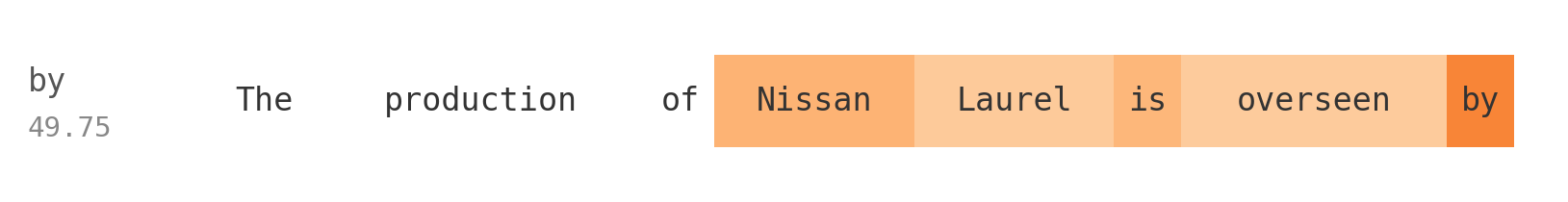} \\
        {\tiny \textsf{Locality}} \\
        \includegraphics[width=\textwidth, height=1.15cm, keepaspectratio, trim=0 0.115cm 0 0.115cm, clip]{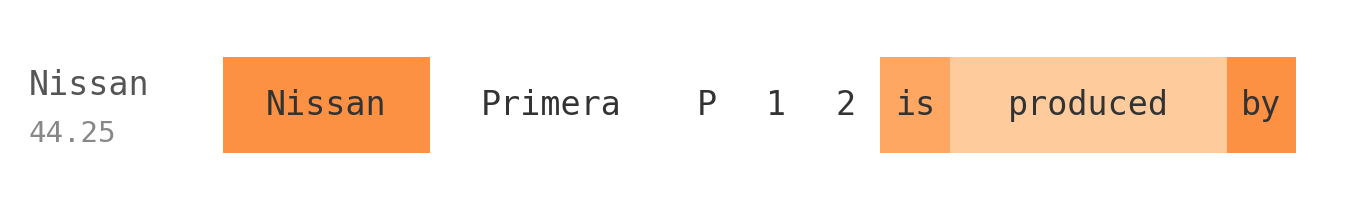}
    \end{minipage}
    \noindent\rule{\textwidth}{\lightrulewidth}
\end{minipage}

\noindent
\begin{minipage}{\textwidth}
    \noindent\rule{\textwidth}{\lightrulewidth}
    \vspace{1mm}
    \begin{minipage}[c]{0.35\textwidth}
        \raggedright
        \textbf{ID:} 19\#15849 \\
        \textbf{Explanation:} \\
        {\footnotesize References to U.S. politics, particularly Barack Obama and governmental structures.}
    \end{minipage}
    \hfill
    \begin{minipage}[c]{0.62\textwidth}
        \centering
        \vspace{3.2mm}
        {\tiny \textsf{Reliability}} \\
        \includegraphics[width=\textwidth, height=1.15cm, keepaspectratio, trim=0 0.115cm 0 0.115cm, clip]{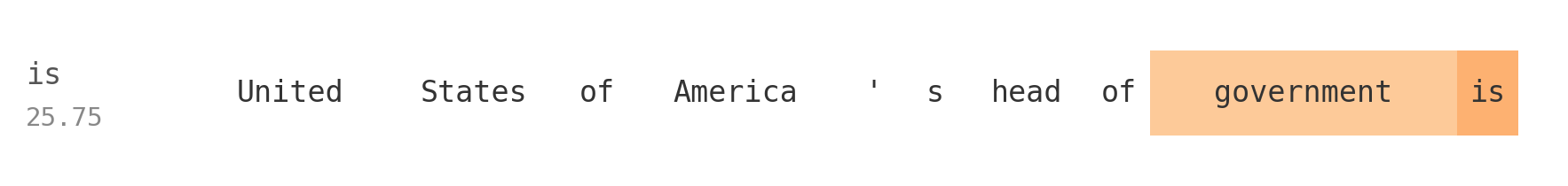} \\
        {\tiny \textsf{Generality}} \\
        \includegraphics[width=\textwidth, height=1.15cm, keepaspectratio, trim=0 0.115cm 0 0.115cm, clip]{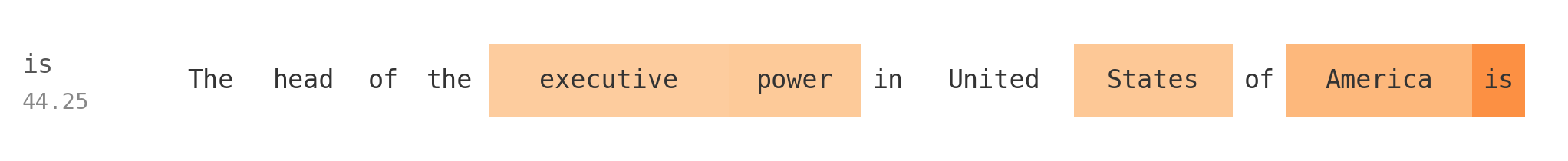} \\
        {\tiny \textsf{Locality}} \\
        \includegraphics[width=\textwidth, height=1.15cm, keepaspectratio, trim=0 0.115cm 0 0.115cm, clip]{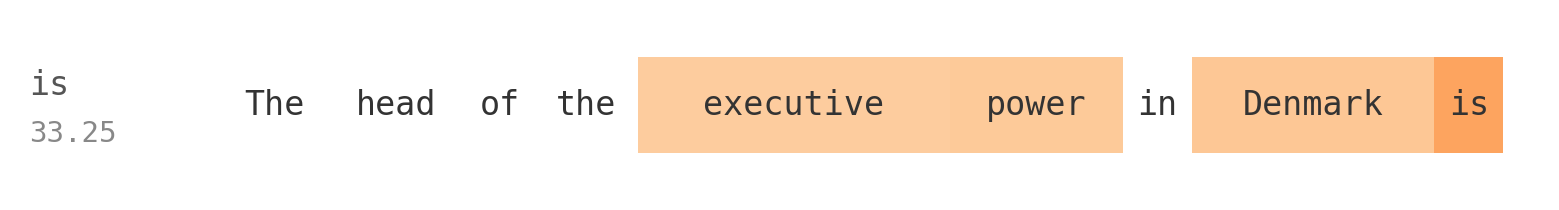}
    \end{minipage}
    \noindent\rule{\textwidth}{\lightrulewidth}
\end{minipage}

\noindent
\begin{minipage}{\textwidth}
    \noindent\rule{\textwidth}{\lightrulewidth}
    \vspace{1mm}
    \begin{minipage}[c]{0.35\textwidth}
        \raggedright
        \textbf{ID:} 20\#15360 \\
        \textbf{Explanation:} \\
        {\footnotesize Mentions of countries and governments, often in a political or geographical context.}
    \end{minipage}
    \hfill
    \begin{minipage}[c]{0.62\textwidth}
        \centering
        \vspace{3.2mm}
        {\tiny \textsf{Reliability}} \\
        \includegraphics[width=\textwidth, height=1.15cm, keepaspectratio, trim=0 0.115cm 0 0.115cm, clip]{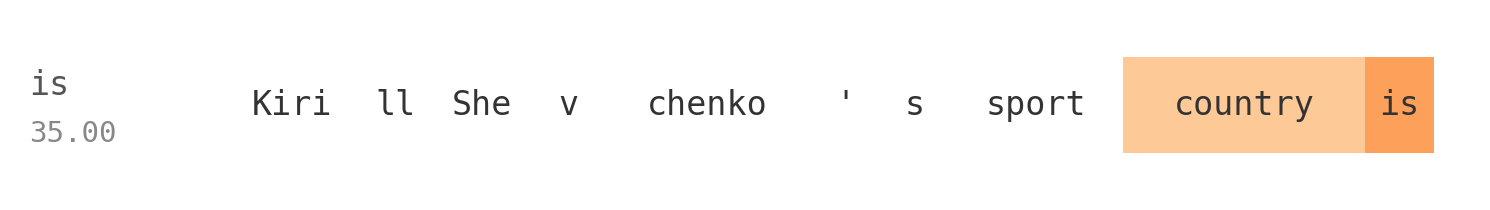} \\
        {\tiny \textsf{Generality}} \\
        \includegraphics[width=\textwidth, height=1.15cm, keepaspectratio, trim=0 0.115cm 0 0.115cm, clip]{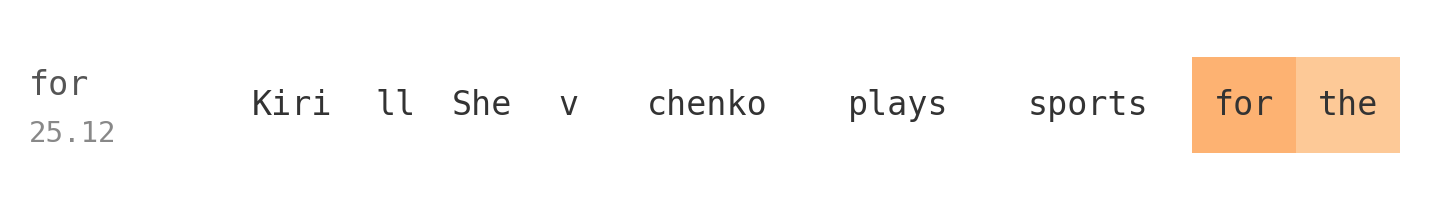} \\
        {\tiny \textsf{Locality}} \\
        \includegraphics[width=\textwidth, height=1.15cm, keepaspectratio, trim=0 0.115cm 0 0.115cm, clip]{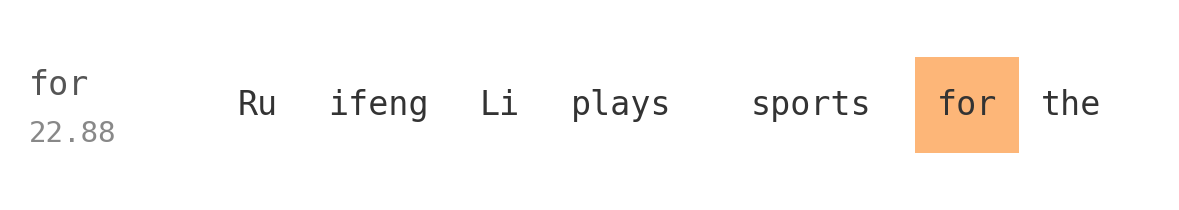}
    \end{minipage}
    \noindent\rule{\textwidth}{\lightrulewidth}
\end{minipage}

\noindent
\begin{minipage}{\textwidth}
    \noindent\rule{\textwidth}{\lightrulewidth}
    \vspace{1mm}
    \begin{minipage}[c]{0.35\textwidth}
        \raggedright
        \textbf{ID:} 19\#1263 \\
        \textbf{Explanation:} \\
        {\footnotesize Political figures or bodies, particularly related to the US government and local offices.}
    \end{minipage}
    \hfill
    \begin{minipage}[c]{0.62\textwidth}
        \centering
        \vspace{3.2mm}
        {\tiny \textsf{Reliability}} \\
        \includegraphics[width=\textwidth, height=1.15cm, keepaspectratio, trim=0 0.115cm 0 0.115cm, clip]{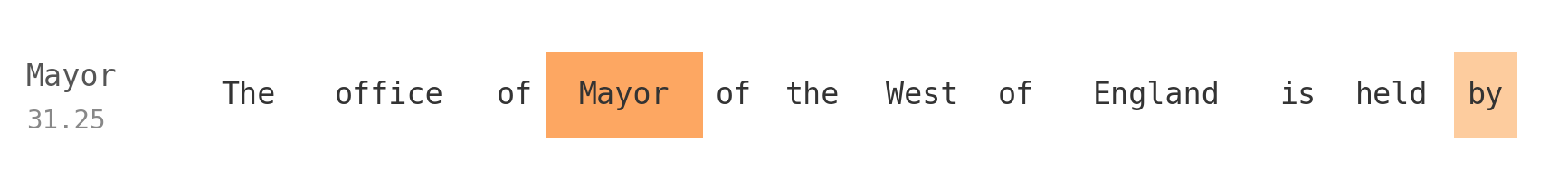} \\
        {\tiny \textsf{Generality}} \\
        \includegraphics[width=\textwidth, height=1.15cm, keepaspectratio, trim=0 0.115cm 0 0.115cm, clip]{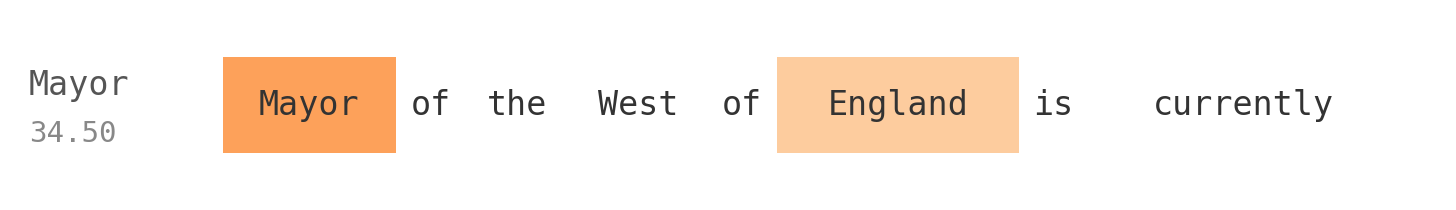} \\
        {\tiny \textsf{Locality}} \\
        \includegraphics[width=\textwidth, height=1.15cm, keepaspectratio, trim=0 0.115cm 0 0.115cm, clip]{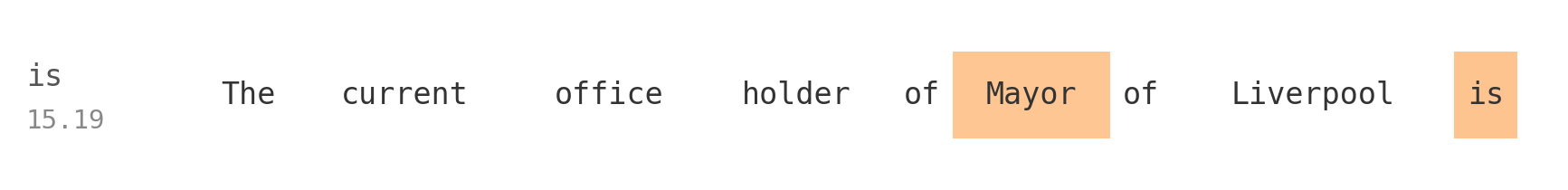}
    \end{minipage}
    \noindent\rule{\textwidth}{\lightrulewidth}
\end{minipage}

\noindent
\begin{minipage}{\textwidth}
    \noindent\rule{\textwidth}{\lightrulewidth}
    \vspace{1mm}
    \begin{minipage}[c]{0.35\textwidth}
        \raggedright
        \textbf{ID:} 19\#14002 \\
        \textbf{Explanation:} \\
        {\footnotesize Mentions of religion, historical figures, and religious affiliations.}
    \end{minipage}
    \hfill
    \begin{minipage}[c]{0.62\textwidth}
        \centering
        \vspace{3.2mm}
        {\tiny \textsf{Reliability}} \\
        \includegraphics[width=\textwidth, height=1.15cm, keepaspectratio, trim=0 0.115cm 0 0.115cm, clip]{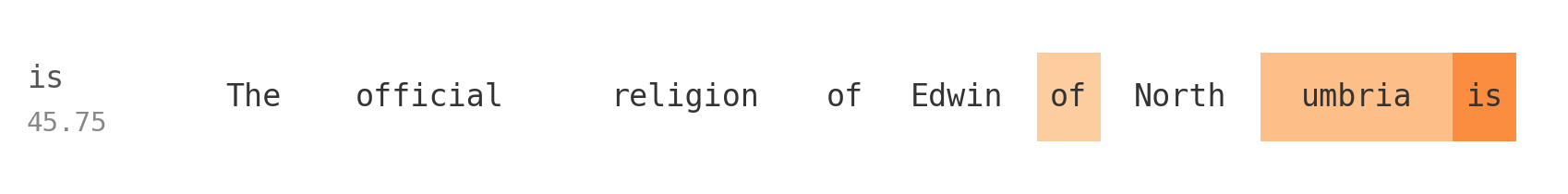} \\
        {\tiny \textsf{Generality}} \\
        \includegraphics[width=\textwidth, height=1.15cm, keepaspectratio, trim=0 0.115cm 0 0.115cm, clip]{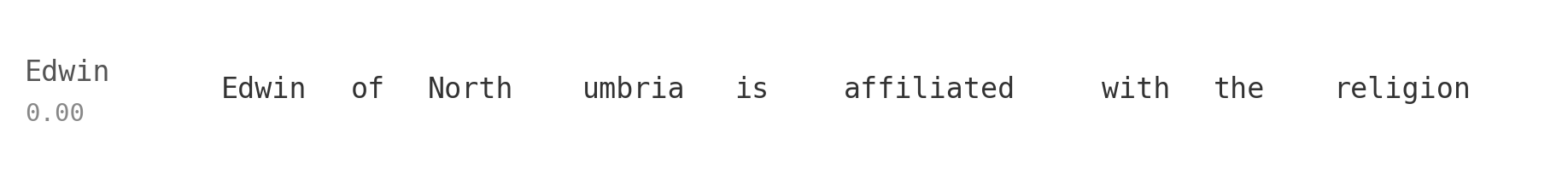} \\
        {\tiny \textsf{Locality}} \\
        \includegraphics[width=\textwidth, height=1.15cm, keepaspectratio, trim=0 0.115cm 0 0.115cm, clip]{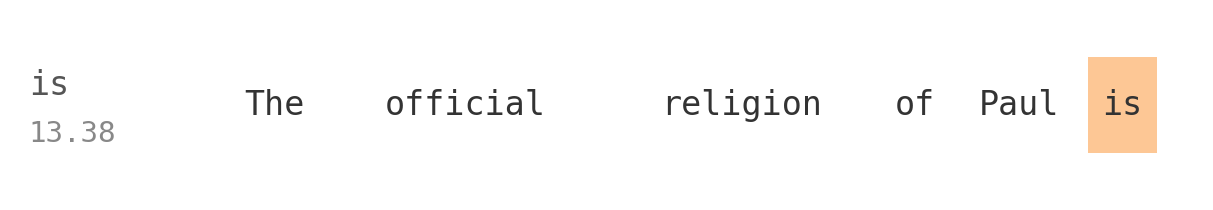}
    \end{minipage}
    \noindent\rule{\textwidth}{\lightrulewidth}
\end{minipage}


\end{document}